\documentclass{article}

\usepackage{PRIMEarxiv}

\usepackage[utf8]{inputenc} 
\usepackage[T1]{fontenc}    
\usepackage{hyperref}       
\usepackage{url}            
\usepackage{booktabs}       
\usepackage{amsfonts}       
\usepackage{nicefrac}       
\usepackage{microtype}      
\usepackage{lipsum}
\usepackage{graphicx}
\usepackage{natbib}  
\usepackage{caption} 
\usepackage{algorithm}
\usepackage{algorithmic}
\usepackage{tabularx}
\usepackage{amsmath,amssymb,amsthm,mathtools}
\usepackage{enumitem}
\usepackage{booktabs}
\usepackage{microtype}
 
\newtheorem{theorem}{Theorem}[section]
\newtheorem{lemma}[theorem]{Lemma}
\newtheorem{proposition}[theorem]{Proposition}

\theoremstyle{definition}
\newtheorem{definition}[theorem]{Definition}
\newtheorem{assumption}[theorem]{Assumption}

\theoremstyle{remark}
\newtheorem{remark}[theorem]{Remark}

\newcommand{\E}{\mathbb E}

\newcommand{\Reg}{\operatorname{Reg}}
\newcommand{\Regsec}{\operatorname{Reg}^{\mathrm{sec}}}
\newcommand{\TV}{\operatorname{TV}}
\newcommand{\conv}{\operatorname{conv}}

\newcommand{\Sset}{\mathcal S}
\newcommand{\Aset}{\mathcal A}

\newcommand{\F}{\mathcal F} 

\newcommand{\DeltaP}{\Delta}   
\newcommand{\phys}{\mathrm{phys}}

\newcommand{\one}{\mathbf 1}

\graphicspath{{media/}}     

\title{Online Security Learning in Cooperative Multi-Agent Systems under Hidden Byzantine Attacks}

\author{
  Ximing Sun, Yue Wang \\
  Department of Electrical and Computer Engineering \\
  University of Central Florida 
}

\begin{document}
\maketitle

\begin{abstract}
We study online cooperative control of a multi-agent system under Byzantine attacks. Namely, an unknown, fixed subset of agents are Byzantine comprised and can stealthily overwrite its own coordinates of the team's planned joint action after observing that plan. The learner observes planned actions, public rewards, and public states, but neither the overwrite nor the executed joint action. Our objective is security: to optimize the team performance against the worst overwrites and achieve the optimal security value. We first show that the attacker's information determines the geometry.  An attacker that observes the planned action induces an exact $(s,a)$-rectangular robust Markov decision process (MDP) whose rows are convex hulls of overwrite-induced public-outcome laws, whereas a blind attacker induces an $s$-rectangular model.  We then identify the information-theoretic limit of security learning, showing that the security regret decomposes exactly into return regret against the response generating the data and a cumulative response gap $D_K$.  Two indistinguishable horizon-one instances force $\Omega(K)$ expected security regret while return regret is zero, showing that dependence on $D_K$ is unavoidable. Finally, we develop a stage-tied robust estimation-to-decisions learner and prove a regret bound of $\widetilde{\mathcal O}\!\left(H^2S\sqrt{AK}\right)+\mathbb E[D_K]$. Our studies thus provide comprehensive theoretical and algorithmic foundations of reliable multi-agent systems under Byzantine attacks. 
\end{abstract}


\section{Introduction}

Cooperative multi-agent systems (MAS) have been studied extensively, expanding from embodied intelligence, including autonomous vehicle fleets and coordinated energy systems \cite{yadav2023comprehensive,hua2025multi, park2022coordinated, yun2022cooperative,shin2019cooperative}, to agentic systems built from foundation models \cite{wu2023autogen, li2023camel, li2026robust}.  Across these domains, the team's value emerges from coordination: heterogeneous agents contribute complementary capabilities and jointly produce outcomes that no individual agent could achieve alone.

Such coordination, however, rests on every teammate executing the action assigned to it, and this premise can fail at deployment: an agent may malfunction, be silently compromised, for example through prompt injection \cite{greshake2023not}, or deliberately deviate from the joint plan.  These failures are Byzantine in the classical sense \cite{lamport2019byzantine}: the affected teammate does not stop but continues to act while appearing functional.  Because its actions directly move the shared dynamical state, standard Byzantine defenses that filter corrupted messages are insufficient; the damage is inflicted by the executed action itself, and no amount of screening undoes it \cite{karimireddy2021learning}.  The threat is not hypothetical: adversarial policies reliably exploit learned agents \cite{gleave2019adversarial}, and a single compromised teammate can severely degrade state-of-the-art cooperative policies \cite{lin2020robustness, li2025attacking, li2026empirical}.  Reliable deployment therefore requires robustness to Byzantine \emph{action overwrites}, the threat model of this paper. 

To address these issues, we propose to study a finite-horizon cooperative team MDP in which an unknown, fixed coalition $B^\star$ observes the team's planned joint action and may replace its own action coordinates before execution. The learner sees the state, planned action, reward, and next state, but not $B^\star$, the overwrite, or the executed action.  We named this a \emph{Byzantine team MDP}. To simplify the problem setting, we consider centralized policy in this work, and aim to develop both theoretical and algorithmic foundations of online learning of such a robust policy under the Byzantine attacked system.



Our contributions can be summarized as answers to three central questions.

\textbf{How does attack information change the decision problem?} We prove an exact public-law reduction.  When the coalition observes the planned action, each state-action row is the convex hull of the laws induced by feasible overwrites, producing an $(s,a)$-rectangular robust MDP \cite{iyengar2005robust}.  When the coalition is blind to the current plan, the coupled choice produces an $s$-rectangular model.  In the insider model, deterministic nonstationary Markov team policies and deterministic Markov worst responses suffice.

\textbf{What can public feedback certify?} We adopt as objective the \emph{security regret}, which charges each deployed policy the gap between the optimal worst-case return and that policy's own worst-case return, and decompose it exactly into the \emph{return regret}, the gap under the attack that actually generated the data, plus a nonnegative \emph{response gap} measuring how far the realized attack is from worst case.  We then construct two single-stage instances that generate identical public transcripts under one shared feasible attack, yet in at least one of them the security regret grows linearly in the number of episodes while the return regret is exactly zero.  The response-gap term is therefore information-theoretically necessary: no algorithm can certify security from public feedback alone against unrestricted realized attacks.

\textbf{What security rate is attainable?} 
We design a learner based on the robust estimation-to-decisions (E2D) framework \cite{foster2021statistical, appel2025regret},  with a stage-tied estimator: one continuation certificate and one occupancy-calibration witness per stage cut the cumulative estimation budget from $\widetilde{\mathcal O}(HS^2)$ to $\widetilde{\mathcal O}(HS)$, where $H$ is the horizon and $S$ the number of public states.  Under the framework's standard approximate decision oracle, we proved the learner attains expected return regret $\widetilde{\mathcal O}(H^2S\sqrt{AK})$ over $K$ episodes with $A$ joint actions, under no assumption on how the coalition plays, and its expected security regret exceeds this by exactly the expected response gap.  The $\sqrt S$ improvement over the direct statewise construction comes entirely from the stage-tied estimator.  The learner never identifies the compromised subset and never estimates overwrites or executed actions.


\subsection{Related Work}
\label{sec:related-work} 

\textbf{Action robustness.}
Action-robust MDPs model replacement or perturbation of a controller's actions \cite{tessler2019action}.  In cooperative MARL, different approaches for such action robustness have been developed. For example, \cite{phan2021resilient} trains against teams with changing agents; \cite{yuan2023robust} trains against a population of limited policy attackers; \cite{bukharin2023robust} uses adversarial regularization for robustness to malicious actions; and \cite{lee2025wolfpack,mcmahan2024roping} studies coordinated attacks and adversarial training, include action manipulation in a general class of online attacks and formulate optimal defense as a stochastic Stackelberg game.  These works primarily study robust training, attack-defense planning, equilibrium convergence, or sample complexity under an explicit attack model, but there is no understandings of learning under direct interactions with hidden compromised agents.

Along the research line of Byzantine robustness, the most related works are \cite{li2024byzantine,nika2024defending}. \cite{li2024byzantine} places a prior over latent Byzantine types and establish robust-equilibrium existence and asymptotic actor-critic convergence. \cite{nika2024defending} further extends the framework to general-sum Markov games and computes adversarially robust Nash equilibria for deployment-time corruption of up to a prescribed number of agents.  We assume no prior, observe neither identity nor execution, and study identifiability and finite-time security regret during online interaction.  Other Byzantine studies are mostly developed in distributed/federated learning and RL, which instead consider malicious messages or updates \cite{chen2023byzantine,xia2023byzantine,fang2025provably,qiao2024br,jiang2025byzantine,lin2022byzantine,zhu2023byzantine}; our attacker changes the physical joint action, so truthful public feedback is itself generated by a strategically selected law.

\textbf{Robust MDPs, robust RL, and corruption robustness.} Classical robust MDPs optimize over a known ambiguity set of state transition probability, with dynamic programming enabled by rectangularity \cite{iyengar2005robust,nilim2005robust,wiesemann2013robust}.  Statistical distributionally robust RL commonly assumes a prescribed ambiguity construction around a nominal kernel and data from that nominal model, and develop sample complexity analysis under different settings \cite{panaganti2022sample,yang2022toward,shi2023curious,wang2021online,lu2024distributionally,he2025sample}.
In our setting, samples instead come from selector-chosen public laws, while the overwrite-induced ambiguity is unknown, thus all these nominal-sampling guarantees cannot be applied. 

Corruption-robust RL instead posits an underlying nominal MDP and measures deviations by a corruption budget \cite{ye2023corruption,zhang2021corruption,lykouris2021corruption}.  Here overwrites may occur at every stage and are part of the security comparator itself, rather than exceptional contamination.  A corruption-budget guarantee therefore does not certify our worst-case overwrite value.

For robust online decision making under general ambiguity sets, \cite{appel2025regret} extend the DEC/E2D framework of \cite{foster2021statistical} to multivalued models in which nature selects adaptively from a convex outcome set.  Their approach results in a regret of $\widetilde{\mathcal O}\!\left(H^2S^{3/2}\sqrt{AK}\right)$ (under normalization to our settings). Our stage-tied estimator reduces the inaccuracy budget by $S$, yielding $\widetilde{\mathcal O}(H^2S\sqrt{AK})$.   

\textbf{Uninformed Markov games and adaptive opponents.} The hidden identity and overwritten actions connect our Byzantine team MDPs to uninformed Markov games, where the opponent's actions are unobservable. \cite{tian2021online} study Markov games with hidden opponent actions and obtain $\widetilde{\mathcal O}\!\left(H^2S^{1/3}A^{2/3}K^{2/3}\right)$ Nash-value regret, which is orderly worse than ours; More importantly, their benchmark evaluates the response actually played, whereas security regret evaluates a worst response to each deployed policy.  \cite{liu2026online} strengthen the former to empirical Nash-value regret and adapt between $\widetilde{\mathcal O}(\sqrt K)$ for fixed opponents and $\widetilde{\mathcal O}(K^{2/3})$ in the worst case.  Related adaptive-opponent work uses external or policy regret under different response conditions \cite{liu2022learning,nguyen2024learning}, which are not adopted in our settings. We defer the more detailed discussion and other related works to Appendix \ref{app:uninformed-games}.

\section{Byzantine Team MDPs}\label{sec:BTMDPs}
\subsection{Model and Security Benchmark}

Let $n\ge2$ be the number of agents, let $\mathcal S$ be a finite
nonempty public state space with fixed initial state $s_1$, and let
$\mathcal A_i$ be agent $i$'s finite nonempty action set.  A centralized
team policy chooses a planned joint action in $\mathcal A\coloneqq\prod_{i=1}^n\mathcal A_i.$ For each stage $h$, state $s$, and executed joint action $a$, the
physical environment has a joint public-outcome kernel $P_h^{\rm phys}(\mathrm d r,s'\mid s,a)
\in\Delta([0,1]\times\mathcal S).$ 
This formulation allows arbitrary dependence between reward and next state;
its marginals give the usual reward and transition kernels.

An unknown, episode-invariant set $B^\star\subseteq[n]$ is compromised.
After the team announces $a$, the coalition may replace its coordinates
by $u\in\mathcal A_{B^\star}\coloneqq
\prod_{j\in B^\star}\mathcal A_j$.  The executed action is then $a\oplus_{B^\star}u\coloneqq(a_{-B^\star},u).
$ 
The overwrite may be randomized and depend on the complete public history
and the announced action, but it is chosen before the current reward and
next state.  The identity $B^\star$, overwrite $u$, and executed action
are hidden.  The team observes only its planned actions, rewards, and
public states.

Suppressing the episode index, the history immediately before the
stage-$h$ planned action is
\begin{equation}
\eta_h=
(s_1,a_1,r_1,s_2,\ldots,a_{h-1},r_{h-1},s_h).
\label{eq:public-preaction-history}
\end{equation}
Let $\mathcal H_h$ be the set of histories.  A behavioral team policy
is
\begin{equation}
\varpi=(\varpi_h)_{h=1}^H,
\qquad
\varpi_h(\cdot\mid\eta_h)\in\DeltaP(\Aset),
\quad \eta_h\in\mathcal H_h.
\label{eq:general-learner-policy}
\end{equation}
Write $\Pi_{\mathrm{hist}}$ for this class.  After observing
$(\eta_h,a)$, a randomized history-dependent overwrite response uses
\begin{equation}
q_h(\cdot\mid\eta_h,a)\in\DeltaP(\Aset_{B^\star}).
\label{eq:general-overwrite-policy}
\end{equation}
Let $\Omega(B^\star)$ contain all collections of these nonanticipating
kernels. This post-action information pattern is the insider model studied in the main paper; Appendix~\ref{app:s-rec} treats a coalition that is blind to the current plan.

For $\varpi\in\Pi_{\mathrm{hist}}$, define its security value as the worst case
\begin{equation}
W(\varpi)
\coloneqq
\inf_{q\in\Omega(B^\star)}
\E^{\varpi,q}\!\left[\sum_{h=1}^H R_h\right].
\label{eq:policy-security-value}
\end{equation}

This is the return $\varpi$ guarantees against every feasible overwrite
response.  The optimal security value is
\begin{equation}
v^\star \coloneqq \sup_{\varpi\in\Pi_{\rm hist}} W(\varpi),
\qquad
\varpi^\star \in \operatorname*{arg\,max}_{\varpi\in\Pi_{\rm hist}}
W(\varpi).
\label{eq:optimal-security-value}
\end{equation}
Theorem~\ref{thm:markov-sufficiency} will prove that the supremum is attained by a deterministic nonstationary Markov policy.

\subsection{Online Security Learning}
\label{sec:online-security}

The same identity $B^\star$ and physical kernel are fixed for $K$
episodes.  In the online problem the learner uses some non-stationary, stochastic Markov policy $\pi_k\in\Pi_{\mathrm{M,stoch}}$, without changing $v^\star$ by
Theorem~\ref{thm:markov-sufficiency}.  Let $\mathcal F_{k-1}$ contain the
public transcript and learner randomness before episode $k$.  A feasible
response then generates the public trajectory.  Our primary objective is the security regret: 
\begin{equation}
\Regsec_K
\coloneqq
\sum_{k=1}^K\bigl\{v^\star-W(\pi_k)\bigr\}.
\label{eq:security-regret}
\end{equation}
If $\Regsec_K=o(K)$, the average worst-case suboptimality of deployed
policies vanishes.

For the learning proof, it is useful to define factual returns under the
episode-$k$ response rule.  Formally, $q_k$ is a predictable feasible
nonanticipating response kernel.  For any queried Markov policy $\pi$, let
$q_k[\pi]$ denote this same response rule paired with $\pi$, not a rule
that observes the policy itself.  We assume directly that $m_k(\pi)$ below
and the trajectory expectations used in the learning proof are Borel in
$\pi$, so they can be integrated against a predictable policy law.  Define
\begin{equation} 
m_k(\pi)
\coloneqq
\mathbb E^{\pi,q_k[\pi]}\!\left[
\sum_{h=1}^H R_{k,h}\,\middle|\,\mathcal F_{k-1}\right],
\label{eq:actual-return}
\end{equation}
with $m_k\coloneqq m_k(\pi_k).$
This permits the responder to adapt to every realized planned action and the
allowed public history, but not to an unobserved policy description.  The
corresponding return regret is only an analytical intermediate:
\begin{equation}
\Reg_K\coloneqq\sum_{k=1}^K(v^\star-m_k).
\label{eq:return-regret}
\end{equation}
It need not be nonnegative when the realized response is weaker than a
worst response.

We can directly show the following connection.
\begin{proposition}[Security decomposition]
\label{prop:security-decomposition}
Define $D_K\coloneqq\sum_{k=1}^K\{m_k-W(\pi_k)\}.$ 
Then $D_K\ge0$ and, pathwise,
\begin{equation}
\Regsec_K=\Reg_K+D_K.
\label{eq:regret-decomposition}
\end{equation}
If the episode-$k$ response is $\epsilon_k$-worst, meaning that a nonnegative $\sigma(\mathcal F_{k-1},\pi_k)$-measurable tolerance satisfies
$m_k\le W(\pi_k)+\epsilon_k$, then
$D_K\le\sum_{k=1}^K\epsilon_k$. If the attacker uses the exact worst responses against $\pi_k$, then $D_K=0$.
\end{proposition}

\begin{proof}
The realized response is feasible, so
$m_k\ge W(\pi_k)$ for every episode.  Thus $D_K\ge0$; expanding the
definitions proves the identity and the approximation claim.
\end{proof}

The response gap is not observed by or supplied to the learner.  It records
the counterfactual security information missing from the realized
trajectory.  Section~\ref{sec:hardness} proves that it can be linear and
unidentifiable under unrestricted responses; Section~\ref{sec:learning}
controls security regret by bounding return regret and adding exactly this
unavoidable gap.

\section{Structure of Byzantine Team MDPs}
\label{sec:geometry}
This section develops the structure of Byzantine team MDPs.  The
information available to the attacker decides the geometry: an insider that
observes the planned action induces an $(s,a)$-rectangular robust MDP
\cite{iyengar2005robust}, while a blind attacker induces an $s$-rectangular
one.  Building on the reduction, we then identify sufficient policy classes
for both the team and the attacker.

\subsection{Exact rectangular public-law reduction}
\label{sec:exact-rectangularity}

Let $\mathcal Y\coloneqq[0,1]\times\Sset$ be the one-step public-outcome
space (reward and next state).  For the fixed identity $B^\star$, define
the public ambiguity row
\begin{equation}
\Gamma^{B^\star}_{h,s,a}
\coloneqq
\conv\!\left\{
P_h^{\phys}(\cdot\mid s,a\oplus_{B^\star}u):
u\in\Aset_{B^\star}
\right\},
\label{eq:fixed-identity-row}
\end{equation}
where $P_h^{\phys}$ is the joint law of $R_h$ and $S_{h+1}$.
Thus $\Gamma^{B^\star}_{h,s,a}$ is a nonempty compact convex subset of
$\Delta(\mathcal Y)$.

The family is state-action rectangular: selecting a row at one
$(h,s,a)$ constrains no other row.  The fixed identity determines which
coordinates are overwritable everywhere, but the absence of an overwrite
budget or commitment lets the coalition select an overwrite distribution
independently at every public history and planned action.  Appendix
\ref{sec:robust-mdps} recalls the robust-MDP definitions.

The following result shows that this rectangular robust-MDP representation is not merely a relaxation: it exactly reproduces the conditional public-outcome laws generated by the operational Byzantine interaction. 

\begin{proposition}[Exact public-law reduction] 
\label{prop:fixed-identity-reduction}
Fix $B^\star\subseteq[n]$.  The conditional public-outcome laws attainable by randomized, history-dependent, nonanticipating overwrite strategies with identity $B^\star$ are exactly the nonanticipating selectors of the rectangular row family $\Gamma^{B^\star}=(\Gamma^{B^\star}_{h,s,a})_{h,s,a}$.  Moreover, for every bounded measurable continuation payoff $f:\mathcal Y\to\mathbb R$, it holds that
\begin{equation}
\inf_{\nu\in\Gamma^{B^\star}_{h,s,a}}\E_\nu [f]
=
\min_{u\in\Aset_{B^\star}}
\E_{P_h^{\phys}(\cdot\mid s,a\oplus_{B^\star}u)}[f].
\label{eq:convex-row-minimum}
\end{equation}
Consequently, if $\{V^\pi_h\}$ is the robust value function of any Markov policy $\pi$ under the $(s,a)$-rectangular robust MDP with uncertainty set $\{\Gamma^{B^\star}\}$, then $W(\pi)=V_1^\pi(s_1)$, even though the infimum in \eqref{eq:policy-security-value} ranges over randomized full-history overwrite strategies.
\end{proposition}
A history-dependent attacker may select different elements of
$\Gamma^{B^\star}_{h,s,a}$ after different public histories ending in the
same $(h,s,a)$; it is therefore formally more adaptive than a selector
that fixes one Markov kernel per row before the episode.  The fixed
identity, however, links none of these selections: the reduction is a
product over public histories, with no cross-stage budget, commitment, or
other coupling.  We will further show that this additional
adaptivity does not lower the worst-case value of a fixed Markov learner policy.

Recall from \eqref{eq:public-preaction-history} that $\eta_h$ denotes the public history immediately before the stage-$h$ planned action.  After the learner announces $a$, the coalition observes $(\eta_h,a)$.   By \eqref{eq:general-overwrite-policy}, a randomized behavioral overwrite strategy may use an arbitrary Borel kernel $q_h(\cdot\mid\eta_h,a)$ at this history.  A nonanticipating selector of the row family in \eqref{eq:fixed-identity-row} is, similarly, a collection of conditional laws
\begin{equation}
\sigma_h(\cdot\mid\eta_h,a)
\in\Gamma^{B^\star}_{h,s_h,a}
\label{eq:public-law-selector}
\end{equation}
that is measurable in the public history and is chosen before the current
outcome.

\subsection{Sufficiency of Markov policies}
\label{sec:markov-sufficiency}

Recall that $\Pi_{\mathrm{hist}}$, defined in \eqref{eq:general-learner-policy}, contains all randomized history-dependent behavioral policies.  Let $\Pi_{\mathrm{M,stoch}}$ be its subclass of randomized nonstationary Markov policies, and let $\Pi_{\mathrm{M,det}}$ be the subclass whose action distributions are point masses.  Episode-level mixtures and other uses of private randomness are included in $\Pi_{\mathrm{hist}}$ through their conditional behavioral action kernels. We then prove the Markov sufficiency under observed planned actions as follows.

\begin{theorem}[Markov sufficiency for an insider]
\label{thm:markov-sufficiency}
For a fixed identity $B^\star$, assume that the coalition observes each
realized planned action before overwriting it and has no cross-stage overwrite
budget or other coupling constraint.  Then:

(1). for every $\pi\in\Pi_{\mathrm{M,stoch}}$, the operational
  worst-case value is $W(\pi)=V_1^\pi(s_1)$, and the infimum over all
  randomized history-dependent overwrite strategies is attained by a deterministic Markov response;

(2).  Markov policies are sufficient for the learner:
  \begin{align}
  v^\star=
    \max_{\pi\in\Pi_{\mathrm{M,stoch}}}W(\pi)
   =\max_{\pi\in\Pi_{\mathrm{M,det}}}W(\pi)
   =V_1^\star(s_1).
  \nonumber
  \end{align}
In particular, an optimal deterministic nonstationary Markov learner policy exists.  
\end{theorem}

\begin{remark} 
\label{rem:scope-markov-sufficiency}
Our result makes two carefully limited assertions.  First, for every fixed behavioral Markov learner policy, there exists a deterministic nonstationary Markov worst-case response, where the response may depend on the current stage, state, and realized planned action.  Second, a deterministic nonstationary Markov learner policy attains the optimal max-min value. For a fixed arbitrary history-dependent learner policy, however, an exact best response may itself need to depend on the history: two public histories can end in the same current state and planned action while inducing different future learner behavior.  Moreover, the responder that generates the online data may remain randomized and history-dependent.

Thus, for the purpose of computing the optimal max-min value, it is without loss of generality to restrict both players to deterministic nonstationary Markov policies. This is a value-level sufficiency statement and does not imply that every history-dependent learner policy or every data-generating response admits a value-equivalent Markov representation.
\end{remark}

If instead the coalition chooses its overwrite before observing the planned
action, the induced robust MDP is $s$-rectangular and randomized Markov
team policies are sufficient; randomization can be essential.  Thus the
attacker's information, not corruption alone, determines rectangularity.
Appendix~\ref{app:s-rec} proves this blind-attacker result.

\section{What Public Feedback Cannot Certify}
\label{sec:hardness}

The rectangular reduction identifies the correct robust decision problem,
but it does not reveal its rows.  More strongly, arbitrary realized
responses may conceal precisely the counterfactual information needed for
security.

\subsection{Why Rectangularity Is Not a Learning Oracle}

Given the reduction of Section~\ref{sec:exact-rectangularity}, one might hope to apply an existing robust RL algorithm directly.  Two gaps prevent this.  First, the learner knows neither $B^\star$ nor the physical public-outcome kernel, hence not the uncertainty set $\Gamma^{B^\star}_{h,s,a}$ that robust dynamic programming takes as given
\cite{iyengar2005robust,nilim2005robust,wiesemann2013robust}, and the observations do not
reveal it: the learner sees the planned action and the public outcome, but not the overwrite or the executed action, so conditional on a reached
post-action history the single observed outcome is drawn from an
attacker-selected point of an unknown convex row.  The physical kernel and
the response are confounded in every sample.
 
Second, the learning literature on robust MDPs predominantly studies
uncertainty sets of a prescribed distributional form, such as a
total-variation, $\chi^2$, KL, or contamination neighborhood around a
nominal kernel, and assumes samples from that nominal kernel through a
generative model, an offline data set, or a distinguished training
environment \cite{wang2021online, shi2023curious, lu2024distributionally,
he2025sample}; the known ambiguity-set construction then turns an estimate
of the nominal kernel into an estimate of the robust model.  Our feedback
has no nominal center: the data-generating law may be any point of the
row, and $\Gamma^{B^\star}_{h,s,a}$, a convex hull of overwrite-induced
kernels, need not equal a neighborhood of any prescribed form.

\subsection{Linear Regret under Unrestricted Responses}

Here \emph{unrestricted} means that the realized responder may use any
feasible randomized, history-dependent, nonanticipating overwrite under the
one fixed coalition.  It does not mean that the identity may switch.
Crucially, the response generating the observations need not be worst-case,
approximately worst-case, or informative about counterfactual overwrites.

\begin{theorem}[Linear security regret under unrestricted feasible responses]
\label{thm:security-impossibility} 
Fix any, possibly randomized, online algorithm and any $K\ge1$.  There
exist two horizon-one, single-state Byzantine team MDPs, indexed by
$\theta\in\{0,1\}$, with two binary-action agents and the same fixed,
even known, identity $B^\star=\{2\}$, together with a single
deterministic realized overwrite strategy feasible in both, such that the
public transcripts generated under the two instances are identical,
$\operatorname{Reg}_K=0$ pathwise in both, and $\max_{\theta\in\{0,1\}}
\mathbb E_\theta\bigl[\operatorname{Reg}^{\rm sec}_K\bigr]
=\max_{\theta\in\{0,1\}}\mathbb E_\theta[D_K]\ \ge\ K/2,$ 
where $\mathbb E_\theta$ denotes expectation over the algorithm's
randomness in instance $\theta$.  Consequently, no algorithm guarantees
$o(K)$ security regret uniformly over feasible realized responses from
public-trajectory feedback alone.
\end{theorem}

Therefore, $D_K$ in our positive result is information-theoretically necessary, not a proof artifact.  The theorem does not say that security learning is impossible whenever the coalition has unrestricted overwrite power.  Instead, arbitrary \emph{data-generating} responses may conceal the worst counterfactual row.  Exact or approximate worst responses, for example, may remove this obstruction, as we shown below. 


\section{Security Learning through Response Gaps}
\label{sec:learning}

This section develops the learning method in two passes.  We first explain what the algorithm does in one episode and identify the role of each object. We then define the statistical estimator and decision rule precisely and derive the regret guarantee.  The central separation is between the \emph{security objective}, which evaluates every deployed policy against its own worst feasible response, and the \emph{realized-response feedback} available to the learner.



Recall the pathwise response-gap identity $\Regsec_K=\Reg_K+D_K$.   The response gap \(D_K\) measures the missing counterfactual information: the difference between the return under the response that actually generated episode \(k\) and the worst return that could have been generated against the same deployed policy.   Theorem~\ref{thm:security-impossibility} shows that \(D_K\) cannot be controlled uniformly from public feedback when the realized responses are arbitrarily uninformative. We will then develop our regret bound in terms of $D_k$.

By Theorem~\ref{thm:markov-sufficiency}, restricting the learner to randomized nonstationary Markov policies does not change \(v^\star\).  We therefore use $\Pi
\coloneqq
\prod_{h=1}^{H}\prod_{s\in\mathcal S}\Delta(\mathcal A)
=\Pi_{\mathrm{M,stoch}}.$

\textbf{Method.}
Our algorithm is designed based on an imported decision module, with our novel estimator design.  The robust estimation-to-decisions (E2D) reduction of \cite{appel2025regret}, building on \cite{foster2021statistical}, converts prediction and optimism bounds into robust return regret.  Our new \emph{stage-tied market} supplies the required prediction bound with one continuation certificate and one calibration witness per layer, instead of per predecessor state, which will be shown efficient. 
In each episode, the learner constructs a policy-indexed predictor, passes it to the E2D module, privately samples and deploys a policy from the returned distribution, and then privately codes the observed rewards to update the market.  The final response-gap identity converts the resulting return guarantee into a security guarantee.

\subsection{Decision-Relevant Representation}
\label{sec:learning-target}

The learner does not estimate the Byzantine identity, overwrite, executed action, or physical kernel.  Instead, it predicts finite public-outcome laws and evaluates only the Bellman inequalities relevant to robust decisions.

\textbf{Private original-scale code.}
Adapting the Bernoulli conversion in \cite{appel2025regret}, after episode \(k\) the learner privately draws ($O_{k,h}\coloneqq(Y_{k,h},S_{k,h+1})\in\mathcal O\coloneqq\{0,1\}\times\mathcal S$)
\begin{equation} 
Y_{k,h}\mid R_{k,h}
\sim\operatorname{Bernoulli}(R_{k,h}/H),
\label{eq:learning-conversion}
\end{equation}
The statistical score is \(HY_{k,h}\), so
\(\mathbb E[HY_{k,h}\mid R_{k,h}]=R_{k,h}\). This step is mainly for technical convenience.  The deferred learner-private sampling and its validity against full-history responses are proved in
Lemmas~\ref{lem:app-deferred-coding} and~\ref{lem:app-conversion}.

\textbf{Stage-tied Bellman certificates.}
For \(v\in[0,H]^{\mathcal S}\) and \(c\in[0,H]\), let $\Psi_H(v,c)
\coloneqq
\left\{\nu\in\Delta(\mathcal O):
\mathbb E_\nu[HY+v(S')]\ge c\right\}$ be a Bellman hyperplane. 

We use set-valued candidates
\begin{equation}
N_h(s,a)=
\begin{cases}
\Psi_H\bigl(v_{h+1},c_h(s)\bigr),&a=a_h(s),\\[2pt]
\Delta(\mathcal O),&a\ne a_h(s),
\end{cases}
\label{eq:learning-stage-class}
\end{equation}
where the recommendation \(a_h(s)\) and threshold \(c_h(s)\) may depend on
\(s\), while the continuation vector \(v_{h+1}\) is shared by every
predecessor state in layer \(h\).  Let
\(\mathfrak H_{\rm st}^{(H)}\) be the \(H\)-bounded members of this class:
each admits a feasible Markov row selector whose conditional task-scale
suffix return is at most \(H\) from every starting triple.  For
\(N\in\mathfrak H_{\rm st}^{(H)}\), write \(v_N^\star\in[0,H]\) for its
optimal robust task-scale value.


The following original-scale and private-coding interface is proved in Appendix~\ref{app:learning-original-scale}.

\begin{proposition}
\label{prop:learning-surrogate}
Let \(V_h^\star\) and \(a_h^\star(s)\) be the optimal robust values and a
Bellman-optimal action for the fixed unknown Byzantine identity.  Taking
\(
a_h=a_h^\star,
v_{h+1}=V_{h+1}^\star,
c_h(s)=V_h^\star(s)
\)
defines an analysis-only candidate
\(N^\circ\in\mathfrak H_{\rm st}^{(H)}\) that contains every feasible coded
row at its recommended actions and satisfies
\(v_{N^\circ}^\star=v^\star\).
\end{proposition}

\textbf{Prediction loss.}
For a predicted row \(\rho\in\Delta(\mathcal O)\) and a compact convex row
\(C\subseteq\Delta(\mathcal O)\), use the directed squared-Hellinger loss $d_C(\rho)
\coloneqq
\min_{q\in C}
\left(1-\sum_{o\in\mathcal O}\sqrt{\rho(o)q(o)}\right).$
A policy-indexed predictor is a map
\(\widehat M:\pi\mapsto M_\pi\), where
\(M_{\pi,h}(s,a)\in\Delta(\mathcal O)\).  Its trajectory loss against
candidate \(N\) is then defined as 
\begin{equation}
L(\widehat M,N,\pi)
\coloneqq
\mathbb E_{M_\pi,\pi}\!\left[
\sum_{h=1}^H
d_{N_h(S_h,A_h)}
\bigl(M_{\pi,h}(S_h,A_h)\bigr)
\right].\nonumber
\end{equation}
This is the loss interface used by the E2D module.  Predictor continuity,
policy coherence, and all measurability details are verified in
Appendix~\ref{app:learning-dec}.

\subsection{Stage-Tied E2D}
\label{sec:learning-algorithm}

Fix \(K\ge2\), set \(\delta=K^{-2}\), and define $\Lambda_K(x)
\coloneqq
\log\!\left(\frac{128H^2A(K+1)^2}{x^2}\right),$ $\beta\coloneqq144HS\Lambda_K(\delta),$ $\varepsilon\coloneqq\sqrt{\beta/K}.$

\textbf{Imported decision module.}
Given earlier predictor--policy pairs, E2D \cite{appel2025regret} retains the localized class
\begin{equation}
\mathcal H_k
\coloneqq
\left\{N\in\mathfrak H_{\rm st}^{(H)}:
\sum_{i=1}^{k-1}
\mathbb E_{\pi\sim p_i}L(\widehat M_i,N,\pi)
\le\beta
\right\}.
\label{eq:learning-localized-class}
\end{equation}

For a queried policy \(\pi\), define its predicted task-scale return as $\widehat v_k(\pi)
\coloneqq
H\,
\mathbb E_{M_{k,\pi},\pi}
\left[\sum_{h=1}^{H}Y_h\right].$ For \(p\in\Delta(\Pi)\), define
\begin{equation}
\begin{split}
Q_k(p)
\coloneqq
\sup_{\substack{
\mu\in\Delta_{\le1}(\mathcal H_k):\\
\mathbb E_{N\sim\mu,\,\pi\sim p}
L(\widehat M_k,N,\pi)\le\varepsilon^2
}}
\mathbb E_{N\sim\mu,\,\pi\sim p}
\left[v_N^\star-\widehat v_k(\pi)\right].
\end{split}
\label{eq:learning-fuzzy-objective}
\end{equation}
Here \(\Delta_{\le1}(\mathcal H_k)\) denotes Borel subprobability
measures, integrated without renormalization.

It is shown in Theorem~4 of \cite{appel2025regret} that the corresponding
unit-scale DEC has coefficient
\(2\sqrt{2(HSA+1)}\).  Our original-scale and terminal-outcome adapters in
Lemma~\ref{lem:app-dec-interface} implies
\begin{equation}
\inf_{p\in\Delta(\Pi)}Q_k(p)
\le
2H\sqrt{2(HSA+1)}\,\varepsilon.
\label{eq:learning-dec-bound}
\end{equation}
Together with the E2D master theorem of
\cite[Theorem~1]{appel2025regret}, this gives the following black-box
implication.  If a fixed value-preserving candidate has cumulative
estimation loss at most \(\beta\) and cumulative optimism at most \(\alpha\),
each with failure probability at most \(\delta\), then an
additive-\(\theta\) decision oracle satisfies
\[
\mathbb E[\Reg_K]
\le
4H\sqrt{2(HSA+1)K\beta}
+\alpha+2K\theta+2HK\delta.
\]
The numerical unit-scale DEC and exact-oracle E2D theorem are imported.
The original task scale, complete terminal outcome, predictable approximate
oracle, and displayed \(2K\theta\) term are established in
Appendices~\ref{app:learning-dec}--\ref{app:learning-e2d}.

Here \(\mathcal F_{k-1}\) denotes the learner's information before episode
\(k\), including its private variables; the precise learner and responder
filtrations are given in Lemma~\ref{lem:app-deferred-coding}.

We further adopt the planning oracle as follows. 
\begin{assumption}[Approximate E2D decision oracle]
\label{ass:decision-oracle}
Given \(\widehat M_k\), \(\mathcal H_k\), and \(\varepsilon\), there exists some oracle that returns an \(\mathcal F_{k-1}\)-measurable
\(p_k\in\Delta(\Pi)\) satisfying
\begin{equation}
Q_k(p_k)
\le
\inf_{p\in\Delta(\Pi)}Q_k(p)+\theta_K,
\qquad
\theta_K\coloneqq K^{-2}.
\label{eq:learning-oracle}
\end{equation}
\end{assumption}
Exact DEC minimization is the oracle in
\cite[Algorithm~1]{appel2025regret}; inexact E2D minimization is standard
\cite[Remark~4.1]{foster2021statistical}.  

\subsubsection{Our new component: the stage-tied market}
\label{sec:learning-estimator}

The statewise market of \cite[Theorem~5]{appel2025regret} provides the
general prediction architecture.  It combines directed-Hellinger fragment
bets, calibration transport, full-support reference bets,
return-overprediction control, backward row minimization, and multiplicative
wealth updates.  As derived in
Appendix~\ref{app:learning-market}, directly adapting its statewise
certificate indexing gives an estimation budget
\(\widetilde{\mathcal O}(HS^2)\).

Our modification preserves the shared continuation vector already present
in the value-preserving surrogate and makes it the statistical unit.  A
layer certificate \(g=(h,a,\bar v,\bar c)\) uses one continuation vector
\(\bar v\) for all predecessor states.  For a point-valued model \(M\), define
\[
x_g^{M,\pi}(s)
\coloneqq
\pi_h(a(s)\mid s)\,
d_{\Psi_H(H\bar v,H\bar c(s))}
\bigl(M_h(s,a(s))\bigr).
\]
Instead of maintaining a separate calibration process for every predecessor
state, in episode \(k\) our market sets
\(M=M_{k,\pi_k}\) and compares the predicted layer average
\(\mathbb E_{M,\pi_k}[x_g^{M,\pi_k}(S_h)]\) with the single factual score
\(x_g^{M,\pi_k}(S_{k,h})\).  Thus the certificate entropy and calibration
charge are paid once per layer.  The exact bettors, posterior update,
supernormalization, and rounding argument are given in
Appendices~\ref{app:learning-market}--\ref{app:learning-inaccuracy}.

\begin{theorem}[Stage-tied estimation]
\label{thm:stage-tied-estimation}
Fix \(\delta\in(0,1)\).  For every fixed
\(N^\circ\in\mathfrak H_{\rm st}^{(H)}\), every randomized,
history-dependent, nonanticipating public-law selector feasible for
\(N^\circ\), and every predictable sequence \((p_k)_{k=1}^K\), the
stage-tied market constructs admissible predictors \(\widehat M_k\) such
that, w.p. \(\geq 1-\delta\),
\begin{equation}
\operatorname{Est}_K
\coloneqq
\sum_{k=1}^{K}
\mathbb E_{\pi\sim p_k}
L(\widehat M_k,N^\circ,\pi)
\le
144HS\Lambda_K(\delta).
\label{eq:learning-estimation-bound}
\end{equation}
\end{theorem}

This result represents our major design. Our stage-tied estimation requires only \(\widetilde{\mathcal O}(HS)\) budget, instead of the 
\(\widetilde{\mathcal O}(HS^2)\) budget of the direct statewise adaptation of \cite{appel2025regret}. This leads to the \(\sqrt S\) improvement in our regret.

\textbf{Complete procedure.} We then combine all components together, and present our algorithm in Alg \ref{alg:stage-tied-e2d}.

\begin{algorithm}[!htb]
\caption{Original-Scale Stage-Tied Robust E2D}
\label{alg:stage-tied-e2d}
\begin{algorithmic}[1]
\REQUIRE episode budget \(K\geq 2\); horizon \(H\); spaces
\(\mathcal S,\mathcal A\)
\STATE set \(\delta\gets K^{-2}\),
\(\beta\gets144HS\Lambda_K(\delta)\),
\(\varepsilon\gets\sqrt{\beta/K}\)
\STATE initialize the stage-tied market of
Appendix~\ref{app:learning-market}
\FOR{\(k=1,\ldots,K\)}
\STATE construct
\(\widehat M_k:\pi\mapsto M_{k,\pi}\) from the current market
\STATE form \(\mathcal H_k\) using
\eqref{eq:learning-localized-class}
\STATE use Assumption~\ref{ass:decision-oracle} to obtain \(p_k\)
\STATE privately sample \(\pi_k\sim p_k\), deploy it, and observe the public
trajectory
\STATE privately draw the bits in \eqref{eq:learning-conversion} and update
the market
\ENDFOR
\end{algorithmic}
\end{algorithm}

The algorithm observes only planned actions, public states, rewards, and its
private bits. The complete filtration and market update are specified in Appendices~\ref{app:learning-original-scale} and
\ref{app:learning-market}.

\begin{table*}[!h]
\centering
\caption{Closest structural baselines under exact realized worst responses ($D_K=0$).  Consequently, Tian et al.'s Nash-value regret and the Appel-Kosoy robust return regret then equal security regret; Liu et al.'s empirical Nash-value regret upper-bounds security regret, with equality only when its additional empirical-comparator gap vanishes.  }
\label{tab:closest-comparison}

\footnotesize
\renewcommand{\arraystretch}{1.08}
\begin{tabularx}{\textwidth}{
  @{}
  p{0.13\textwidth}
  p{0.28\textwidth}
  p{0.20\textwidth}
  X
  @{}
}
\toprule
Work
& Response and feedback
& Benchmark when $D_K=0$
& Guarantee / stated factors
\\
\midrule

\cite{tian2021online}
& Hidden opponent action in a Markov game
& Security regret
  $=$ Nash-value regret
& $\widetilde{\mathcal O}\!\left(
    H^2 S^{1/3} A^{2/3} K^{2/3}
  \right)$
\\

\cite{liu2026online}
& Hidden Markov opponent with variation and switch measures
& Security regret
  $\le$ empirical Nash-value regret
& $\widetilde{\mathcal O}_{H,S,A}\!\left(
    \min\!\left\{
      \sqrt K+(C_{\rm var}K)^{1/3},
      \sqrt{L_{\rm sw}K}
    \right\}
  \right)$
\\

\cite{appel2025regret}
& Adaptive selector from a multivalued robust model;
  direct statewise adaptation
& Security regret
  $=$ robust return regret
& $\widetilde{\mathcal O}\!\left(
    H^2 S^{3/2}\sqrt{AK}
  \right)$
\\

This work
& Hidden identity and executed action;  full-history responses allowed
& Security regret
& $    \widetilde{\mathcal O}\!\left(
     H^2 S\sqrt{AK}
   \right)$
\\

\bottomrule
\end{tabularx}
\end{table*}

\subsection{Security-Regret Guarantee}
Proposition~\ref{prop:learning-surrogate} supplies
the fixed analysis witness \(N^\circ\).  Our stage-tied theorem gives
\(\beta=\widetilde{\mathcal O}(HS)\).  The return-overprediction bettor,
adapted from \cite[Theorems~2 and~5]{appel2025regret} and analyzed on our
task scale in Lemma~\ref{lem:app-optimism}, gives $\alpha
\le
8H^2\left\{
\sqrt{K\log(32/\delta)}
+\log(32/\delta)
\right\}.$
The imported E2D black box therefore yields return regret 
\begin{align}
    &\widetilde{\mathcal O}\!\left(H^2S\sqrt{AK}\right)\\
    &=\widetilde{\mathcal O}\!\left(
\underbrace{H}_{\text{task scale}}
\sqrt{
\underbrace{HSA}_{\text{DEC dimension}}\,
\underbrace{HS}_{\text{our estimation budget}}\,
K}
\right).\nonumber
\end{align}

Finally, our exact identity
\(\Regsec_K=\Reg_K+D_K\) converts this intermediate control into the desired
security guarantee. Specifically, denote $\mathcal C_K
\coloneqq
114H^2S\sqrt{AK\Lambda_K(K^{-2})}$, and we have the following regret bound.

\begin{theorem} 
\label{thm:stage-tied-regret}
Suppose the response rule induces Borel trajectory expectations as a function of the queried policy, as in Section~\ref{sec:online-security}.  Run Algorithm~\ref{alg:stage-tied-e2d}; for \(K\ge2\), suppose Assumption~\ref{ass:decision-oracle} holds.  Then
\begin{equation}
\mathbb E[\Regsec_K]
\le
\min\!\left\{HK,\,
\mathcal C_K+\mathbb E[D_K]\right\}.
\label{eq:learning-security-regret}
\end{equation}
If the episode-\(k\) response is ex-post \(\epsilon_k\)-worst for the
deployed policy, meaning
\(m_k-W(\pi_k)\le\epsilon_k\) almost surely for a nonnegative
\(\mathcal F_{k-1}\vee\sigma(\pi_k)\)-measurable tolerance, then
\begin{equation}
\mathbb E[\Regsec_K]
\le
\min\!\left\{HK,\,
\mathcal C_K+\sum_{k=1}^K\mathbb E[\epsilon_k]\right\}.
\label{eq:learning-approx-security-regret}
\end{equation}
Particularly, exact ex-post worst responses give \(D_K=0\) and
\(\mathbb E[\Regsec_K]\le\mathcal C_K\).  The factual-return
guarantee is
\begin{equation}
\mathbb E[\Reg_K]
\le
\min\!\left\{HK,\,\mathcal C_K\right\}.
\label{eq:learning-return-regret}
\end{equation}
\end{theorem}

The detailed substitution is in
Appendix~\ref{app:learning-final-proof}.  The E2D reduction controls
\eqref{eq:learning-return-regret}; the passage to
\eqref{eq:learning-security-regret} is our problem-specific response-gap
conversion.

 \subsection{Comparison with the Closest Guarantees}
\label{sec:learning-comparison}

We then compare our results with the most related works, which are summarized in Table~\ref{tab:closest-comparison}. It specializes every row to the case in which the response realized against each deployed policy is an exact ex-post worst
response.  Then \(D_K=0\), and the displayed return-based bounds can be read
as security-regret guarantees.  

\textbf{Appel--Kosoy.}
\cite[Corollary~3]{appel2025regret} give
\(\widetilde{\mathcal O}(HS^{3/2}\sqrt{AK})\) for 1-bounded tabular
robust MDPs.  Their condition asks for a selector whose conditional suffix
reward is at most one from every starting triple; it does not normalize every
realized episode.  Applying their estimator to our private Bernoulli
certificate and scaling only the task value by \(H\) gives $\widetilde{\mathcal O}(H^2S^{3/2}\sqrt{AK}).$ The stage-tied estimator replaces the \(\widetilde{\mathcal O}(HS^2)\) inaccuracy budget of this direct statewise specialization by \(\widetilde{\mathcal O}(HS)\), improving the leading state dependence by \(\sqrt S\).  Like Appel and Kosoy, we do not establish polynomial-time DEC minimization; Assumption~\ref{ass:decision-oracle} states the measurable \(K^{-2}\)-approximate oracle needed here, and Appendix~\ref{app:learning-e2d} records its additive cost.

\textbf{Uninformed Markov games.}
The Tian and Liu rows use the same turn-based augmentation and require
episodewise Markov post-action responses to import their finite-state rates.
Appendix~\ref{app:uninformed-games} proves the exact value and regret
identities, translates the game dimensions, and explains when either
benchmark yields a security guarantee.

\section{Conclusion}

We introduced the Byzantine team MDP, in which an unknown, fixed subset of
a cooperative team may overwrite its planned action coordinates after
seeing the joint plan.  The attacker's information determines robust
geometry: an insider induces an exact $(s,a)$-rectangular public-law model, while a blind attacker induces an $s$-rectangular one.  Our results also identify the fundamental limitation of learning security from public feedback.  The observed trajectory reveals the return under the realized response, but not how close that response is to the worst feasible one.  This unavoidable ambiguity is captured by the response gap $D_K$. Under a standard approximate E2D decision-oracle condition, our stage-tied learner attains $E\!\left[\Reg_K^{\rm sec}\right]
    \le
    \widetilde{\mathcal O}\!\left(H^2S\sqrt{AK}\right)
    + \E[D_K]$.  The bound separates the statistical cost of
learning from the information-theoretic cost of certifying security.
These findings clarify both when security can be learned and what public feedback alone cannot certify. We leave the problem of settling the minimax optimal regret bound as future work.

\bibliography{references}
\bibliographystyle{unsrt}


\appendix
\onecolumn

\section{Preliminaries of Robust MDPs}\label{sec:robust-mdps}

Throughout this appendix, fix a finite state space $\mathcal S$, a
finite action space $\mathcal A$, a horizon $H$, a fixed initial
state $s_1$, and the one-step public-outcome space
$\mathcal Y\coloneqq[0,1]\times\mathcal S$, whose elements are pairs
$(r,s')$ of a reward and a next state.
 
\textbf{Row families and selectors.}
An $(s,a)$-rectangular row family is a collection
$\Gamma=(\Gamma_{h,s,a})_{h\in[H],s\in\mathcal S,a\in\mathcal A}$ in
which each row $\Gamma_{h,s,a}$ is a nonempty compact convex subset of
$\Delta(\mathcal Y)$, and the choice of an element in one row places no
constraint on the choice in any other row.  Nature interacts with a team
policy through selectors.  A \emph{nonanticipating selector} is a
collection $\sigma=(\sigma_h)_{h=1}^H$ of Borel kernels
$\sigma_h(\cdot\mid\eta_h,a)\in\Gamma_{h,s_h,a}$, one conditional law
per public history $\eta_h$ (ending in state $s_h$) and realized
planned action $a$, chosen before the current outcome.  Write
$\Sigma(\Gamma)$ for the class of all such selectors.  The
\emph{selector value} of a team policy $\varpi$ is
\[
W_\Gamma(\varpi)\coloneqq
\inf_{\sigma\in\Sigma(\Gamma)}
\mathbb E^{\varpi,\sigma}\Bigl[\sum_{h=1}^{H}R_h\Bigr],
\]
the worst case over all history-dependent, within-row selections.  Note
that the selector observes the realized planned action before choosing
its row element; this post-action timing is what produces the per-action
inner infimum in the recursions below.
 
\textbf{Robust value of a Markov policy.}
For a nonstationary Markov policy $\pi=(\pi_h)_{h=1}^H$ with
$\pi_h(\cdot\mid s)\in\Delta(\mathcal A)$, the robust value function is
defined by the backward recursion
\begin{equation}
\begin{aligned}
V^\pi_{H+1}(s)&=0,\\
V^\pi_h(s)&=\sum_{a\in\mathcal A}\pi_h(a\mid s)
\inf_{\nu\in\Gamma_{h,s,a}}
\mathbb E_\nu\bigl[R+V^\pi_{h+1}(S')\bigr].
\end{aligned}
\label{eq:policy-bellman}
\end{equation}
Each inner infimum is attained: the row is compact and
$\nu\mapsto\mathbb E_\nu[R+V^\pi_{h+1}(S')]$ is linear and continuous.
When the row is a polytope, a minimizer can be chosen among its vertices.
 
The next lemma is the standard dynamic-programming identity for
rectangular families \cite{iyengar2005robust, nilim2005robust}, stated
in the selector form used in this paper.
 
\begin{lemma}[Rectangular dynamic programming]
\label{lem:app-rect-dp}
For every $(s,a)$-rectangular compact row family $\Gamma$ and every
nonstationary Markov policy $\pi$,
\[
W_\Gamma(\pi)=V^\pi_1(s_1).
\]
Moreover, the infimum in $W_\Gamma(\pi)$ is attained by the stagewise
selector that plays, at every $(h,s,a)$ and every history, one fixed
minimizer $\nu^\star_{h,s,a}$ of
$\nu\mapsto\mathbb E_\nu[R+V^\pi_{h+1}(S')]$; when the rows are
polytopes, each $\nu^\star_{h,s,a}$ can be taken to be a vertex.  In
particular, history-dependent selection does not lower the value of a
Markov policy below its stagewise worst case.
\end{lemma}
 
\begin{proof}
For a public history $\eta_h$ ending in state $s$, let
$W_h(\eta_h)$ denote the infimum, over selectors restricted to stages
$t\ge h$, of $\mathbb E[\sum_{t=h}^H R_t\mid\eta_h]$.  We prove
$W_h(\eta_h)=V^\pi_h(s)$ by backward induction; the case $h=1$ is the
claim.  At $h=H+1$ both sides are zero.  Assume the identity at
$h+1$.
 
Fix any selector $\sigma$ for stages $t\ge h$.  Conditioning on the
planned action $a\sim\pi_h(\cdot\mid s)$ and on the realized outcome
$(R,S')\sim\sigma_h(\cdot\mid\eta_h,a)$, the conditional continuation
value of $\sigma$ from $\eta_{h+1}=(\eta_h,a,R,S')$ is at least
$W_{h+1}(\eta_{h+1})$, which equals $V^\pi_{h+1}(S')$ by the
induction hypothesis, for every history prefix.  Hence
\begin{align*}
 \mathbb E^{\sigma}\Bigl[\sum_{t=h}^H R_t\,\Big|\,\eta_h\Bigr] \ge\sum_{a}\pi_h(a\mid s)
\inf_{\nu\in\Gamma_{h,s,a}}
\mathbb E_\nu\bigl[R+V^\pi_{h+1}(S')\bigr]
=V^\pi_h(s),
\end{align*}
where the inequality also uses that the stage-$h$ selection at each
realized $a$ is one feasible element of $\Gamma_{h,s,a}$.  Taking the
infimum over $\sigma$ gives $W_h(\eta_h)\ge V^\pi_h(s)$.
 
Conversely, let $\sigma^\star$ play the fixed minimizer
$\nu^\star_{t,s',a'}$ at every stage $t\ge h$, state $s'$, and
action $a'$, regardless of the history; this is feasible because
rectangularity places no constraint across rows, and measurable because
it is history-independent.  A second backward induction computes its
value exactly: conditional on $\eta_h$,
\begin{align*}
\mathbb E^{\sigma^\star}\Bigl[\sum_{t=h}^H R_t\,\Big|\,\eta_h\Bigr]=\sum_a\pi_h(a\mid s)\,
\mathbb E_{\nu^\star_{h,s,a}}\bigl[R+V^\pi_{h+1}(S')\bigr] =V^\pi_h(s).
\end{align*}
Hence $W_h(\eta_h)\le V^\pi_h(s)$, which closes the induction.  The
vertex refinement follows from linearity of the objective on a polytope.
\end{proof}
 
\textbf{Operator facts.}
Define, for $V\colon\mathcal S\to\mathbb R$,
\[
(\mathcal T_h V)(s)\coloneqq
\max_{a\in\mathcal A}\,
\inf_{\nu\in\Gamma_{h,s,a}}
\mathbb E_\nu\bigl[R+V(S')\bigr].
\]
 
\begin{lemma}[Monotonicity and nonexpansiveness]
\label{lem:app-operator}
If $V\le V'$ pointwise, then
$\mathcal T_hV\le\mathcal T_hV'$ pointwise, and
$\|\mathcal T_hV-\mathcal T_hV'\|_\infty\le\|V-V'\|_\infty$.  The same
holds for the policy operator obtained by replacing the maximum with
$\sum_a\pi_h(a\mid s)$.
\end{lemma}
 
\begin{proof}
For every fixed $\nu$,
$\mathbb E_\nu[R+V(S')]\le\mathbb E_\nu[R+V'(S')]\le
\mathbb E_\nu[R+V(S')]+\|V-V'\|_\infty$.  Infima, maxima, and convex
combinations of maps preserve both pointwise ordering and uniform
distance.
\end{proof}
 
\textbf{Optimal value and greedy optimality.}
The optimal robust value function is defined by
\begin{equation}
\begin{aligned}
V^\star_{H+1}(s)&=0,\\
V^\star_h(s)&=\max_{a\in\mathcal A}\,
\inf_{\nu\in\Gamma_{h,s,a}}
\mathbb E_\nu\bigl[R+V^\star_{h+1}(S')\bigr].
\end{aligned}
\label{eq:optimal-bellman}
\end{equation}
 
\begin{lemma}[Deterministic greedy optimality]
\label{lem:app-greedy}
Let $\pi^\star$ be any deterministic nonstationary Markov policy with
$\pi^\star_h(s)\in\arg\max_{a}\inf_{\nu\in\Gamma_{h,s,a}}
\mathbb E_\nu[R+V^\star_{h+1}(S')]$.  Then
\[
V^{\pi^\star}_h=V^\star_h
=\max_{\pi\in\Pi_{\mathrm{M,stoch}}}V^\pi_h
\quad\text{pointwise for every }h,
\]
where the maximum ranges over all randomized nonstationary Markov
policies.  In particular, randomization does not increase the robust
value under $(s,a)$-rectangularity.
\end{lemma}
 
\begin{proof}
For any randomized Markov $\pi$, backward induction with
Lemma~\ref{lem:app-operator} gives $V^\pi_h\le V^\star_h$: assuming
$V^\pi_{h+1}\le V^\star_{h+1}$,
\[
V^\pi_h(s)
\le\sum_a\pi_h(a\mid s)\inf_{\nu\in\Gamma_{h,s,a}}
\mathbb E_\nu\bigl[R+V^\star_{h+1}(S')\bigr]
\le V^\star_h(s),
\]
since a convex combination never exceeds the maximum.  Conversely,
backward induction along $\pi^\star$ gives
$V^{\pi^\star}_h=V^\star_h$: assuming equality at $h+1$, the
recursion \eqref{eq:policy-bellman} at the greedy action
reproduces \eqref{eq:optimal-bellman}.
\end{proof}
 
Combining Lemmas~\ref{lem:app-rect-dp} and~\ref{lem:app-greedy}: for rectangular compact rows, the security level
$\max_\pi W_\Gamma(\pi)$ over Markov policies equals
$V^\star_1(s_1)$, it is attained by a deterministic greedy policy, and
a worst-case response can be taken stagewise, at a vertex when the rows
are polytopes.  Proposition~\ref{prop:fixed-identity-reduction} shows that
the Byzantine rows form exactly such a family, generated by
$\{P^{\rm phys}_h(\cdot\mid s,a\oplus_{B^\star}u):
u\in\mathcal A_{B^\star}\}$; hence every vertex is one of these laws.  Moreover,
Theorem~\ref{thm:markov-sufficiency} extends the identity
$v^\star=V^\star_1(s_1)$ to the operational value defined over all
history-dependent stochastic team policies.
 
\textbf{$s$-rectangular families.}
An $s$-rectangular family couples the rows within a state: for each
$(h,s)$ a nonempty compact convex set
$\Gamma_{h,s}\subseteq\Delta(\mathcal Y)^{\mathcal A}$ of
action-indexed profiles $\nu=(\nu_a)_{a\in\mathcal A}$, with no
constraint across states or stages.  The robust value of a Markov policy
and the optimal value satisfy
\begin{equation}
\begin{aligned}
V^\pi_h(s)&=\inf_{\nu\in\Gamma_{h,s}}
\sum_{a\in\mathcal A}\pi_h(a\mid s)\,
\mathbb E_{\nu_a}\bigl[R+V^\pi_{h+1}(S')\bigr],\\
V^\star_h(s)&=\max_{\delta\in\Delta(\mathcal A)}\,
\inf_{\nu\in\Gamma_{h,s}}
\sum_{a\in\mathcal A}\delta(a)\,
\mathbb E_{\nu_a}\bigl[R+V^\star_{h+1}(S')\bigr],
\end{aligned}
\label{eq:app-s-rect}
\end{equation}
with the infimum now outside the sum over actions: nature commits to one
profile for the whole state before the realized action, which is the
simultaneous-move timing.  Optimal Markov policies may then be genuinely
randomized, and deterministic policies are suboptimal in general
\cite{wiesemann2013robust}.  This is the geometry induced by the blind
attacker in Appendix~\ref{app:s-rec}, in contrast with the per-action infimum of \eqref{eq:policy-bellman} induced by the insider.

\section{Byzantine Team MDPs and Uninformed Markov Games}
\label{app:uninformed-games}

This appendix formalizes three distinct connections: the game representation, the regret comparators, and the interpretation of existing uninformed-game rates in the Byzantine team MDP.  The two assumptions that appear below play different roles.  Episodic Markov responses are needed to import the finite-state rates of \cite{tian2021online} and \cite{liu2026online}; exact or approximately worst responses are instead needed to convert their regret benchmarks into security guarantees. Neither assumption implies the other.

\subsection{Value-preserving turn-based representation}

Fix the episode-invariant identity $B^\star$, and for a team policy $\pi$ and overwrite response $q$, write
\[
J(\pi,q)
\coloneqq
\mathbb E^{\pi,q}\!\left[\sum_{h=1}^H R_h\right].
\]
The hidden-action feature of our feedback matches an uninformed Markov game: the learner observes its own planned action, reward, and next state, but not the coalition's overwrite. Nevertheless, the timing is different:  In the simultaneous formulation of \cite{tian2021online}, the minimizing action cannot depend on the maximizing player's current action, whereas our coalition observes the realized plan before overwriting it.

This timing difference has a finite-state turn-based representation.  For each physical stage $h$, introduce the two state layers
\[
\widetilde{\mathcal S}_{2h-1}=\mathcal S,
\qquad
\widetilde{\mathcal S}_{2h}=\mathcal S\times\mathcal A.
\]
At layer $2h-1$, the team chooses $a\in\mathcal A$, the coalition has only a dummy action, the reward is zero, and the game moves deterministically from $s$ to $(s,a)$.  At layer $2h$, the team has only a dummy action, the coalition chooses $u\in\mathcal A_{B^\star}$, and the public outcome is drawn according to
\[
P_h^{\rm phys}
\bigl(\mathrm d r,s'\mid s,a\oplus_{B^\star}u\bigr).
\]
Because the planned action is part of the intermediate state, the coalition may condition on it even though the two players formally act simultaneously at each augmented substep.

Let $\Omega_{\rm M}(B^\star)$ denote the nonstationary Markov responses $q_h(\cdot\mid s,a)$.  Markov policies correspond through
\begin{equation}
\widetilde\pi_{2h-1}(a\mid s)=\pi_h(a\mid s),
\qquad
\widetilde q_{2h}(u\mid(s,a))=q_h(u\mid s,a),
\label{eq:app-augmented-policy-map}
\end{equation}
with dummy policies on the other layers.  The transformed parameters are
\begin{equation}
\widetilde H=2H,\qquad
\widetilde S=\max\{S,SA\}=SA,\qquad
\widetilde A=A,\qquad
\widetilde B=|\mathcal A_{B^\star}|.
\label{eq:app-augmented-dimensions}
\end{equation}

\begin{proposition}[Value-preserving augmentation]
\label{prop:app-uninformed-augmentation}
For every Markov pair $(\pi,q)$, the augmented pair in
\eqref{eq:app-augmented-policy-map} induces the same physical public-trajectory
law and the same total return.  Consequently,
\[
\widetilde V_1^{\widetilde\pi,\widetilde q}(s_1)=J(\pi,q),
\qquad
\widetilde V_1^\star(s_1)
=
\max_{\pi\in\Pi_{\rm M,stoch}}
\min_{q\in\Omega_{\rm M}(B^\star)}J(\pi,q)
=v^\star.
\]
\end{proposition}

\begin{proof}
At the first substep of stage $h$, the augmented game records precisely the
planned action drawn by $\pi_h(\cdot\mid s)$.  At the second substep, it
draws the overwrite from the same row $q_h(\cdot\mid s,a)$ and then uses the
same physical public-outcome kernel.  Induction over $h$ therefore gives
equality of the physical public-trajectory laws.  Each inserted reward is
zero, so episode returns are also equal.  Maximizing and minimizing over the
corresponding Markov policies and applying
Theorem~\ref{thm:markov-sufficiency} proves the value identity.
\end{proof}

The finite-state game dynamics also represent arbitrary behavioral policies: such policies may condition on the augmented game's observed history.  The limitation concerns theorem applicability, not representation. \cite{tian2021online,liu2026online} formulate each episode policy as a nonstationary Markov kernel. To encode a general response $q_h(\cdot\mid\eta_h,a)$ as such a kernel, one would have to replace the intermediate state $(s,a)$ by $(\eta_h,a)$.  This Markovization is exponentially large and can be uncountable when rewards are continuously valued.  Their tabular rates hence apply directly only to the episodic Markov response subclass, although the game-theoretic representation itself is more general.

There is one further technical qualification.  Both cited papers formally use a deterministic bounded reward function and a next-state kernel, whereas our physical kernel may jointly randomize reward and next state.  For Markov policy pairs, replacing a joint row by its mean reward and next-state marginal preserves $\mathbb E[R_h+V_{h+1}(S_{h+1})]$, and hence preserves every value and regret identity below.  Applying their learning theorems to the full stochastic feedback law additionally invokes the standard bounded stochastic-reward martingale extension (or, literally, assumes deterministic
rewards). 

\subsection{Tian et al.: Nash-value regret is return regret}

We then carefully discuss the regret notions.

Suppose the episode-$k$ response is Markov, and condition on the public
transcript before that episode.  Then
$m_k=J(\pi_k,q_k)$ in the notation of
equation~\eqref{eq:actual-return}.  The minimax-value regret of
\cite{tian2021online}, subsequently called Nash-value regret, compares the game value with the \emph{expected value of the policy pair actually used};
it is not regret against the sampled episode return.  Proposition
\ref{prop:app-uninformed-augmentation} gives the exact identity
\begin{equation}
\operatorname{NR}^{\rm post}_K
\coloneqq
\sum_{k=1}^K
\left\{\widetilde V_1^\star(s_1)
-\widetilde V_1^{\widetilde\pi_k,\widetilde q_k}(s_1)\right\}
=\sum_{k=1}^K\{v^\star-m_k\}
=\Reg_K.
\label{eq:app-tian-return-regret}
\end{equation}
Note that we do not assume any condition on the hidden player's response.  A benign response can make $m_k>W(\pi_k)$, and can even make a return-regret summand negative, while
the deployed policy remains vulnerable to another feasible response.

For each episode define the nonnegative response gap
\[
d_k\coloneqq m_k-W(\pi_k),
\qquad D_K=\sum_{k=1}^K d_k.
\]
Combining \eqref{eq:app-tian-return-regret} with the security decomposition
gives
\begin{equation}
\boxed{
\Regsec_K=\operatorname{NR}^{\rm post}_K+D_K.}
\label{eq:app-tian-security-identity}
\end{equation}
Thus \cite{tian2021online}'s regret equals security regret if and only if $D_K=0$ (the hidden player plays the best-response, i.e., the worst-case).
Since every $d_k\ge0$, this is equivalent to the realized response attaining
$W(\pi_k)$ in every episode.  If it is $\epsilon_k$-worst, then
\begin{equation}
0\le d_k\le\epsilon_k,
\qquad
\Regsec_K
\le \operatorname{NR}^{\rm post}_K+
\sum_{k=1}^K\epsilon_k.
\label{eq:app-tian-approx-security}
\end{equation}
More generally, Nash-value no-regret implies security no-regret whenever
$D_K=o(K)$.  Numerical equality in
\eqref{eq:app-tian-security-identity} may hold for a history-dependent exact
worst response, but importing \cite{tian2021online}'s learning rate additionally requires the data-generating responses to satisfy the Markov-policy and finite-state conditions above.

In the large-$K$ branch of their theorem, substituting
\eqref{eq:app-augmented-dimensions} gives, with high probability,
\begin{equation}
\operatorname{NR}^{\rm post}_K
=\widetilde{\mathcal O}\!\left(
H^2S^{1/3}A^{2/3}K^{2/3}\right),
\qquad
K\gtrsim 8H^3SA^2.
\label{eq:app-tian-transformed-rate}
\end{equation}
The bound is independent of $|\mathcal A_{B^\star}|$; a fixed unknown
identity is absorbed into the fixed unknown game and need not be identified
by the learner.  An episode-varying identity, in contrast, changes the
underlying game and is outside this reduction.  Under exact worst responses,
\eqref{eq:app-tian-transformed-rate} is also a security-regret bound; under
$\epsilon_k$-worst responses one adds $\sum_k\epsilon_k$; and under
unrestricted responses it controls only the intermediate return regret. However, this regret has a worse order in $K$ compared to ours. 

\subsection{Liu et al.: the empirical-comparator gap}

The empirical comparator of \cite{liu2026online} is statewise
rectangular.  It does not restrict the minimizing player to one complete policy among $q_1,\ldots,q_K$.  At each augmented response state $(h,s,a)$, define
\begin{equation}
\mathcal Q^{\rm emp}_{K,h,s,a}
\coloneqq
\operatorname{conv}\!\left\{
q_{k,h}(\cdot\mid s,a):k\in[K]
\right\}.
\label{eq:app-empirical-response-row}
\end{equation}
Convexification leaves the Bellman minimum unchanged because its objective is
linear in the response distribution.  Let $\Omega_K^{\rm emp}(B^\star)$
contain all Markov responses whose row at each $(h,s,a)$ belongs to
$\mathcal Q^{\rm emp}_{K,h,s,a}$.  This class may splice rows taken from
different episode policies at different states.  Its fixed-initial-state
empirical value is
\begin{equation}
v_K^{\rm emp}
\coloneqq
\max_{\pi\in\Pi_{\rm M,stoch}}
\inf_{q\in\Omega_K^{\rm emp}(B^\star)}J(\pi,q),
\qquad
\Delta_K^{\rm emp}\coloneqq v_K^{\rm emp}-v^\star.
\label{eq:app-empirical-value}
\end{equation}
The empirical response class restricts the full Markov response class, so
$\Delta_K^{\rm emp}\ge0$.  For a fixed initial state, \cite{liu2026online}'s
empirical Nash-value regret therefore satisfies
\begin{align}
\operatorname{ENR}_K
&\coloneqq
\sum_{k=1}^K\{v_K^{\rm emp}-m_k\}
\notag\\
&=\operatorname{NR}^{\rm post}_K+K\Delta_K^{\rm emp}
=\Reg_K+K\Delta_K^{\rm emp}.
\label{eq:app-enr-return-relation}
\end{align}
If initial states vary, the last term is replaced by the sum of the
state-dependent empirical-value gaps.

Together, the three regret notions obey
\begin{equation}
\boxed{
\operatorname{NR}^{\rm post}_K=\Reg_K,
\qquad
\operatorname{ENR}_K=\Reg_K+K\Delta_K^{\rm emp},
\qquad
\Regsec_K=\Reg_K+D_K.}
\label{eq:app-three-regret-identities}
\end{equation}
Equivalently,
\begin{equation}
\Regsec_K
=\operatorname{ENR}_K+D_K-K\Delta_K^{\rm emp}
\le \operatorname{ENR}_K+D_K.
\label{eq:app-enr-security-relation}
\end{equation}
This distinction between equality and guarantee transfer is important.  Under
exact worst responses, $D_K=0$, so
\begin{equation}
\Regsec_K
=\operatorname{ENR}_K-K\Delta_K^{\rm emp}
\le\operatorname{ENR}_K.
\label{eq:app-enr-exact-security}
\end{equation}
Thus an upper bound on empirical Nash-value regret \emph{is} a security-regret
upper bound under exact worst responses, even though the two quantities need
not be equal.  With $\epsilon_k$-worst responses,
\begin{equation}
\Regsec_K
\le\operatorname{ENR}_K+
\sum_{k=1}^K\epsilon_k.
\label{eq:app-enr-approx-security}
\end{equation}
Under exact worst responses, equality
$\operatorname{ENR}_K=\Regsec_K$ holds if and only if
$v_K^{\rm emp}=v^\star$.  Without exact worst responses, the purely
algebraic equality condition is
$D_K=K\Delta_K^{\rm emp}$, but cancellation of these two conceptually
different gaps is not a useful security assumption.

A transparent sufficient condition for $v_K^{\rm emp}=v^\star$ is
\emph{value completeness}.  Suppose a globally minimax Markov response
$q^\dagger$ satisfies
\[
\max_{\pi\in\Pi_{\rm M,stoch}}J(\pi,q^\dagger)=v^\star
\]
and, for every $(h,s,a)$,
\begin{equation}
q_h^\dagger(\cdot\mid s,a)
\in\mathcal Q^{\rm emp}_{K,h,s,a}.
\label{eq:app-value-complete-coverage}
\end{equation}
Then the empirical opponent can select $q^\dagger$, giving
$v_K^{\rm emp}\le\max_\pi J(\pi,q^\dagger)=v^\star$; the reverse inequality
was proved above.  Complete coverage of all pure overwrite rows at every
$(h,s,a)$ is a stronger sufficient condition.  Exact worst responses to
the deployed policies $\pi_1,\ldots,\pi_K$ do not by themselves imply
\eqref{eq:app-value-complete-coverage}: they control factual policies, while
the empirical value optimizes over counterfactual learner policies and may
need response rows never used against the deployed sequence.

If the opponent is fixed, $q_1=\cdots=q_K=q$, then
$\Omega_K^{\rm emp}$ contains only the rows of $q$, and
\begin{equation}
\operatorname{ENR}_K
=
\max_{\pi}
\sum_{k=1}^K\{J(\pi,q)-J(\pi_k,q)\}.
\label{eq:app-fixed-opponent-external-regret}
\end{equation}
This is ordinary external regret in the fixed MDP induced by $q$, as \cite{liu2026online} emphasize.  It is not automatically security regret.  Stationarity
does not make $q$ worst for each deployed policy, and worst-responsiveness
does not make the empirical response rows value-complete.  Both gaps vanish
only under additional conditions.

For completeness, define the post-action nonstationarity quantities
\begin{align}
C_{\rm post}
&\coloneqq
\sum_{h=1}^H\sum_{k=1}^K
\TV\!\left(
q_{k,h}(\cdot\mid S_{k,h},A_{k,h}),
q_h^{\rm emp,\star}(\cdot\mid S_{k,h},A_{k,h})
\right),
\notag\\
L_{\rm post}
&\coloneqq
1+\sum_{k=1}^{K-1}\mathbf 1\{q_k\ne q_{k+1}\},
\label{eq:app-liu-nonstationarity}
\end{align}
where $q^{\rm emp,\star}$ is an empirical minimax row selector.  The first
quantity is an on-trajectory total-variation measure, not the security
response gap $D_K$.  For an oblivious sequence of episodewise Markov
responses, substituting \eqref{eq:app-augmented-dimensions} into \cite{liu2026online}'s
bound gives, up to logarithmic factors,
\begin{align}
\operatorname{ENR}_K
\le\widetilde{\mathcal O}\!\Bigl(
 H^2SA+\min\bigl\{
H^{5/2}A\sqrt{SK}
+H^{7/3}A^{2/3}(SKC_{\rm post})^{1/3}, 
H^{5/2}A\sqrt{SL_{\rm post}K}
\bigr\}\Bigr).
\label{eq:app-liu-transformed-rate}
\end{align}
If $q_k$ may adapt to the pre-episode transcript, their adaptive-opponent
theorem is the relevant result.  It preserves the same dependence on
$K,C_{\rm post},L_{\rm post}$, with additional state factors.  In the
original augmented parameters, it has the form
\begin{equation}
\begin{aligned}
\widetilde{\mathcal O}\!\Bigl(
 \widetilde H^2\widetilde S^{3/2}
+\min\Bigl\{\widetilde S^{3/4}
\sqrt{\widetilde H^3\iota_{\rm aug}K} +(\iota_{\rm aug}\widetilde H^5\widetilde S
KC_{\rm post})^{1/3}, \widetilde S^{3/4}
\sqrt{L_{\rm post}\widetilde H^3\iota_{\rm aug}K}
\Bigr\}\Bigr).
\end{aligned}
\label{eq:app-liu-adaptive-rate}
\end{equation}
where
$\iota_{\rm aug}
=\Theta(\widetilde H^2\widetilde A
\log(\widetilde H K\widetilde A\widetilde S/\delta))$.
Suppressing fixed game dimensions, both results interpolate as
\[
\widetilde{\mathcal O}_{H,S,A}\!\left(
\min\left\{
\sqrt K+(C_{\rm post}K)^{1/3},
\sqrt{L_{\rm post}K}
\right\}\right).
\]
A fixed response has $C_{\rm post}=0$ and $L_{\rm post}=1$, giving a
$\sqrt K$-type empirical/external-regret rate.  Exact worst responses may
nevertheless switch every episode, so $D_K=0$ does not imply small
$C_{\rm post}$ or $L_{\rm post}$.  Conversely, a fixed benign response
can have $C_{\rm post}=0$ while $D_K=\Theta(K)$.  Response
nonstationarity and response worstness are therefore orthogonal properties.

In summary, \cite{tian2021online}'s theorem controls $\Reg_K$ under Markov response
feedback, and \cite{liu2026online}'s theorem controls the stronger empirical benchmark
$\operatorname{ENR}_K$.  Either result becomes a security guarantee when
the cumulative response gap is controlled.  Exact equality with security
regret requires $D_K=0$ for \cite{tian2021online}; for \cite{liu2026online} it additionally
requires $v_K^{\rm emp}=v^\star$.

\subsection{Other hidden-action formulations.}
Classical stochastic games with imperfect monitoring allow each player to observe the public state and its own action but only a stochastic signal of the opponent's action \cite{rosenberg2003maxmin}.  These works establish
long-run value existence for a known game, rather than finite-time learning
guarantees.  Partially observable Markov games provide a broader learning
model in which both the state and other players' actions may be hidden
\cite{liu2022sample}.  Against arbitrary opponents, their maximin-value regret
has the same game-value-minus-realized-return form as the return regret above.
Under local-only feedback, however, learning this comparator can be
exponentially hard even against a fixed known opponent; the positive
sublinear-regret result assumes that, after every episode, all players reveal
their observations and actions.  Thus their negative result concerns an
additional latent-state identification problem absent from our public-state
model, whereas their positive result uses strictly richer feedback.  As with
Nash-value regret, their maximin-value regret becomes security regret only
when the realized responses are exact or approximate worst responses.

Payoff-based decentralized Markov-game methods also avoid observing opponent
actions \cite{sayin2022fictitious,chen2023finite,
ouhamma2026learning}.  Related decentralized regret guarantees are available
when all players follow a prescribed slowly varying learning procedure
\cite{erez2023regret}.  These results rely on opponent stationarity or
coordinated self-play and target equilibrium convergence or deviation regret;
they therefore do not cover one-sided deployment against an arbitrary
Byzantine response.  At horizon one, \cite{ito2026adversarial} study
pure-strategy maximin regret under uninformed bandit feedback.  Their
comparator agrees with the value of the one-stage post-action game, but their
opponent moves simultaneously and cannot condition on the learner's current
sampled action.  Moreover, their realized-return guarantee certifies security
only when the realized opponent actions are worst or approximately worst
responses.

\section{Proofs for Section \ref{sec:geometry}}

\begin{proposition}[Restatement of Proposition~\ref{prop:fixed-identity-reduction}] 
Fix $B^\star\subseteq[n]$.  The conditional public-outcome laws attainable by randomized, history-dependent, nonanticipating overwrite strategies with identity $B^\star$ are exactly the nonanticipating selectors of the rectangular row family $\Gamma^{B^\star}=(\Gamma^{B^\star}_{h,s,a})_{h,s,a}$.  Moreover, for every bounded measurable continuation payoff $f:\mathcal Y\to\mathbb R$, it holds that
\begin{equation}
\inf_{\nu\in\Gamma^{B^\star}_{h,s,a}}\E_\nu [f]
=
\min_{u\in\Aset_{B^\star}}
\E_{P_h^{\phys}(\cdot\mid s,a\oplus_{B^\star}u)}[f].
\end{equation}
Consequently, if $\{V^\pi_h\}$ is the robust value function of any Markov policy $\pi$ under the $(s,a)$-rectangular robust MDP with uncertainty set $\{\Gamma^{B^\star}\}$, then $W(\pi)=V_1^\pi(s_1)$, even though the infimum in \eqref{eq:policy-security-value} ranges over randomized full-history overwrite strategies.
\end{proposition} 

\begin{proof}[Proof of Proposition~\ref{prop:fixed-identity-reduction}]
We prove separately the two directions of the public-law equivalence, the
linear-minimization identity, and the Bellman consequence.

\emph{From overwrites to public-law selectors.}
Fix $h$, a public history $\eta_h$ ending in state $s$, and a realized
planned action $a$.  If the coalition uses the kernel in
\eqref{eq:general-overwrite-policy}, then for every Borel set
$E\subseteq\mathcal Y$, the conditional public-outcome law is
\begin{align}
Q_h^q(E\mid\eta_h,a)
&=
\sum_{u\in\Aset_{B^\star}}
q_h(u\mid\eta_h,a)
P_h^{\phys}(E\mid s,a\oplus_{B^\star}u).
\label{eq:overwrite-induced-public-law}
\end{align}
The coefficients are nonnegative and sum to one.  Hence
$Q_h^q(\cdot\mid\eta_h,a)$ belongs to
$\Gamma^{B^\star}_{h,s,a}$.  It is measurable in the history because
$q_h$ is a probability kernel and the overwrite-action set is finite.
Thus every randomized, history-dependent, nonanticipating overwrite strategy
induces a selector satisfying \eqref{eq:public-law-selector}.

\emph{From public-law selectors to overwrites.}
Enumerate
$\Aset_{B^\star}=\{u_1,\ldots,u_m\}$.  For each fixed row define the affine
map
\begin{align}
T_{h,s,a}:\DeltaP(\Aset_{B^\star})
&\longrightarrow\Gamma^{B^\star}_{h,s,a},\nonumber\\
T_{h,s,a}(\lambda)
&\coloneqq
\sum_{j=1}^m\lambda_j
P_h^{\phys}(\cdot\mid s,a\oplus_{B^\star}u_j).
\label{eq:row-mixture-map}
\end{align}
The definition of the convex hull says exactly that $T_{h,s,a}$ is
surjective.  Therefore, for every selector value
$\sigma_h(\cdot\mid\eta_h,a)$, at least one coefficient vector
$\lambda\in\DeltaP(\Aset_{B^\star})$ satisfies
\begin{equation}
T_{h,s,a}(\lambda)
=\sigma_h(\cdot\mid\eta_h,a).
\label{eq:selector-mixture-representation}
\end{equation}
There is no hidden measurability gap here.  Under the weak topology,
$T_{h,s,a}$ is a continuous surjection between compact metric spaces, so a
standard measurable-selection theorem supplies a Borel right inverse.  Fix
one such right inverse for each of the finitely many rows and apply it to
$\sigma_h(\cdot\mid\eta_h,a)$.  The resulting coefficient vector defines a
Borel kernel $q_h(\cdot\mid\eta_h,a)$.  Drawing the overwrite from this
kernel implements \eqref{eq:selector-mixture-representation}.  The conditional
public-outcome kernels are therefore equal at every history.  Repeated
conditioning over $h=1,\ldots,H$ shows that, under any learner policy, the
two descriptions generate the same law of the complete public trajectory.

\emph{Linear minimization over a row.}
Let $f:\mathcal Y\to\mathbb R$ be bounded and Borel measurable.  Every
$\nu\in\Gamma^{B^\star}_{h,s,a}$ has a representation
$
\nu=\sum_u\lambda_u
P_h^{\phys}(\cdot\mid s,a\oplus_{B^\star}u)
$
for some $\lambda\in\DeltaP(\Aset_{B^\star})$.  By linearity of
integration,
\begin{align}
\E_\nu f=
\sum_{u\in\Aset_{B^\star}}\lambda_u
\E_{P_h^{\phys}(\cdot\mid s,a\oplus_{B^\star}u)}f\ge
\min_{u\in\Aset_{B^\star}}
\E_{P_h^{\phys}(\cdot\mid s,a\oplus_{B^\star}u)}f.
\label{eq:mixture-objective-lower-bound}
\end{align}
Conversely, every pure overwrite law is itself an element of the convex hull.
Taking a point mass on a minimizing overwrite gives the reverse inequality,
and the minimum exists because $\Aset_{B^\star}$ is finite.  This proves
\eqref{eq:convex-row-minimum}.

\emph{Bellman evaluation against full-history attacks.}
Given any behavioral learner policy $\vartheta$, any overwrite strategy
$q$, and any fixed public history $\eta_h$, recursively applying the
learner kernel, overwrite kernel, and physical outcome kernel from stage $h$
onward defines a unique continuation trajectory law.  Denote expectation
under this law by $\E_{\eta_h}^{\vartheta,q}$.  This expectation is defined
for every history, including a history having probability zero under earlier
play; at reached histories it agrees with a version of the usual conditional
expectation.

Fix a randomized nonstationary Markov policy $\pi$.  For a continuation of
an overwrite strategy $q$, let
\begin{equation}
J_h^{\pi,q}(\eta_h)
\coloneqq
\E_{\eta_h}^{\pi,q}\!\left[\sum_{t=h}^H R_t\right]
\label{eq:strategy-continuation-return}
\end{equation}
denote the continuation return in the game beginning at $\eta_h$.  We prove
by backward induction that
for every feasible $q$ and every $\eta_h$ ending in $s$,
\begin{equation}
J_h^{\pi,q}(\eta_h)\ge V_h^\pi(s),
\label{eq:any-response-bellman-lower-bound}
\end{equation}
and that equality is attained by one deterministic nonstationary Markov
overwrite strategy.

For the remainder of the proof, abbreviate
\begin{equation}
P^u_{h,s,a}
\coloneqq P_h^{\phys}(\cdot\mid s,a\oplus_{B^\star}u).
\label{eq:primitive-overwrite-row-shorthand}
\end{equation}

At $h=H+1$, both sides are zero.  Suppose the claim holds at stage
$h+1$.  Conditioning successively on the planned action, overwrite, and
public outcome gives
\begin{align}
J_h^{\pi,q}(\eta_h)
&=\sum_{a\in\Aset}\pi_h(a\mid s)
  \sum_{u\in\Aset_{B^\star}}q_h(u\mid\eta_h,a)
\nonumber\\[-0.3em]
&\quad\cdot
\E_{P^u_{h,s,a}}
\!\left[R+J_{h+1}^{\pi,q}(\eta_h,a,R,S')\right]
\nonumber\\
&\ge\sum_{a\in\Aset}\pi_h(a\mid s)
\min_{u\in\Aset_{B^\star}}
\E_{P^u_{h,s,a}}
\!\left[R+V_{h+1}^\pi(S')\right]
\nonumber\\
&=V_h^\pi(s).
\label{eq:bellman-evaluation-induction}
\end{align}
The first inequality uses the induction hypothesis after every possible next
public history and then uses that an average is at least its smallest term.
The final equality follows from
\eqref{eq:convex-row-minimum} and \eqref{eq:policy-bellman}.

For every finite triple $(h,s,a)$, choose
\begin{equation}
u_h^\pi(s,a)
\in\operatorname*{arg\,min}_{u\in\Aset_{B^\star}}
\E_{P^u_{h,s,a}}
\!\left[R+V_{h+1}^\pi(S')\right].
\label{eq:markov-overwrite-best-response}
\end{equation}
If the coalition always uses
$U_h=u_h^\pi(S_h,A_h)$, then the induction hypothesis holds with equality
at stage $h+1$, and the selected overwrite attains the current minimum.
Consequently every inequality in
\eqref{eq:bellman-evaluation-induction} is an equality.  Backward induction
therefore shows that this deterministic nonstationary Markov response attains
$V_h^\pi(s)$ after every history ending in $s$.  This proves the Bellman
claim in Proposition \ref{prop:fixed-identity-reduction}, including its validity against
arbitrary randomized full-history overwrite strategies.
\end{proof}

\begin{theorem}[Restatement of Theorem~\ref{thm:markov-sufficiency}] 
For a fixed identity $B^\star$, assume that the coalition observes each
realized planned action before overwriting it and has no cross-stage overwrite
budget or other coupling constraint.  Then:
\begin{enumerate}[label=\textup{(\roman*)},leftmargin=1.55em,itemsep=0.25em]
  \item for every $\pi\in\Pi_{\mathrm{M,stoch}}$, the operational
  worst-case value is $W(\pi)=V_1^\pi(s_1)$, and the infimum over all
  randomized history-dependent overwrite strategies is attained by a deterministic Markov response;

  \item Markov policies are sufficient for the learner:
  \begin{align}
  v^\star=
    \max_{\pi\in\Pi_{\mathrm{M,stoch}}}W(\pi)
   =\max_{\pi\in\Pi_{\mathrm{M,det}}}W(\pi)
   =V_1^\star(s_1).
  \end{align}
In particular, an optimal deterministic nonstationary Markov learner policy exists. 
\end{enumerate}
\end{theorem}

\begin{proof}
Part~\textup{(i)} was established in the Bellman portion of the proof of
Proposition \ref{prop:fixed-identity-reduction}.  It remains to prove the policy-class
equalities in part~\textup{(ii)}.

Define
\begin{equation}
Q_h^\star(s,a)
\coloneqq
\min_{u\in\Aset_{B^\star}}
\E_{P_h^{\phys}(\cdot\mid s,a\oplus_{B^\star}u)}
\!\left[R+V_{h+1}^\star(S')\right].
\label{eq:optimal-pure-overwrite-q}
\end{equation}
By \eqref{eq:convex-row-minimum} and \eqref{eq:optimal-bellman},
$V_h^\star(s)=\max_aQ_h^\star(s,a)$.  Since the action sets are finite,
choose
\begin{align}
a_h^\star(s)&\in\operatorname*{arg\,max}_{a\in\Aset}Q_h^\star(s,a),
\label{eq:optimal-markov-learner-action}\\
u_h^\star(s,a)&\in\operatorname*{arg\,min}_{u\in\Aset_{B^\star}}
\E_{P_h^{\phys}(\cdot\mid s,a\oplus_{B^\star}u)}
\!\left[R+V_{h+1}^\star(S')\right].
\label{eq:universal-markov-overwrite}
\end{align}
Let $\pi^\star$ always choose $a_h^\star(S_h)$.  Applying part~\textup{(i)}
and backward induction in \eqref{eq:policy-bellman} and \eqref{eq:optimal-bellman} gives
\begin{equation}
W(\pi^\star)=V_1^\star(s_1).
\label{eq:deterministic-markov-attains-value}
\end{equation}

We next show that no history-dependent randomized policy can do better.  Fix
an arbitrary $\varpi\in\Pi_{\mathrm{hist}}$, and let the coalition use the
single deterministic Markov strategy
$U_h=u_h^\star(S_h,A_h)$.  We claim that after every public history
$\eta_h$ ending in state $s$,
\begin{equation}
\E_{\eta_h}^{\varpi,u^\star}\!\left[\sum_{t=h}^H R_t\right]
\le V_h^\star(s).
\label{eq:universal-response-upper-bound}
\end{equation}
The claim is proved by backward induction.  It is immediate at $H+1$.
Assuming it at $h+1$, condition on the action drawn from
$\varpi_h(\cdot\mid\eta_h)$.  The induction hypothesis applies after every
next public history, so
\begin{align}
&\E_{\eta_h}^{\varpi,u^\star}\!\left[\sum_{t=h}^H R_t\right]
\nonumber\\
&\quad\le
\sum_{a\in\Aset}\varpi_h(a\mid\eta_h)
\E_{P_h^{\phys}(\cdot\mid
s,a\oplus_{B^\star}u_h^\star(s,a))}
\!\left[R+V_{h+1}^\star(S')\right]
\nonumber\\
&\quad=
\sum_{a\in\Aset}\varpi_h(a\mid\eta_h)Q_h^\star(s,a)
\le\max_{a\in\Aset}Q_h^\star(s,a)
=V_h^\star(s).
\label{eq:history-policy-upper-induction}
\end{align}
The penultimate inequality holds because a convex combination cannot exceed
its largest term.  At the initial state,
\begin{align}
\inf_{q\in\Omega(B^\star)}
\E^{\varpi,q}\!\left[\sum_{h=1}^H R_h\right]
&\le
\E^{\varpi,u^\star}\!\left[\sum_{h=1}^H R_h\right]
\le V_1^\star(s_1).
\label{eq:any-history-policy-upper-bound}
\end{align}
Combining this upper bound with
\eqref{eq:deterministic-markov-attains-value} and the inclusions
$
\Pi_{\mathrm{M,det}}
\subseteq\Pi_{\mathrm{M,stoch}}
\subseteq\Pi_{\mathrm{hist}}
$
proves the claim.
\end{proof}


\section{$s$-Rectangular Reduction}\label{app:s-rec}
We now consider a simultaneous-move information structure.  At a public
pre-action history $\eta_h$ ending in state $s_h$, the learner draws a
planned joint action
\[
  A_h\sim\varpi_h(\cdot\mid\eta_h),
\]
while, simultaneously, the Byzantine coalition draws an overwrite
\[
  U_h\sim q_h(\cdot\mid\eta_h),
  \qquad
  q_h(\cdot\mid\eta_h)\in\DeltaP(\Aset_{B^\star}).
\]
The two draws use conditionally independent private randomness.  The
coalition may know the learner's policy, and hence the distribution
$\varpi_h(\cdot\mid\eta_h)$, but it does not observe the realized planned
action $A_h$ before choosing $U_h$.  The executed action is
$A_h\oplus_{B^\star}U_h$, after which the public outcome is drawn from the
physical kernel.  Let $\Omega_{\rm sim}(B^\star)$ denote the class of all
randomized, history-dependent, nonanticipating overwrite strategies of this
form.

For a learner policy $\varpi\in\Pi_{\mathrm{hist}}$, define its
simultaneous-move security value by
\begin{equation}
W_{\rm sim}(\varpi)
\coloneqq
\inf_{q\in\Omega_{\rm sim}(B^\star)}
\E^{\varpi,q}\!\left[\sum_{h=1}^H R_h\right],
\qquad
v_{\rm sim}^\star
\coloneqq
\sup_{\varpi\in\Pi_{\mathrm{hist}}}W_{\rm sim}(\varpi).
\label{eq:simultaneous-security-values}
\end{equation}

\textit{The induced row tuples.}
Let $\mathcal Y\coloneqq[0,1]\times\Sset$.  For each
$(h,s,u)\in[H]\times\Sset\times\Aset_{B^\star}$, define the tuple of
public-outcome rows
\begin{equation}
\mathbf P^{u}_{h,s}
\coloneqq
\left(
P_h^{\phys}(\cdot\mid s,a\oplus_{B^\star}u)
\right)_{a\in\Aset}
\in\DeltaP(\mathcal Y)^{\Aset}.
\label{eq:simultaneous-pure-row-tuple}
\end{equation}
The simultaneous-move ambiguity set at $(h,s)$ is
\begin{align}
\Gamma^{B^\star,\rm sim}_{h,s}
&\coloneqq
\conv\!\left\{
\mathbf P^{u}_{h,s}:u\in\Aset_{B^\star}
\right\}
\nonumber\\
&=
\left\{
\left(
\sum_{u\in\Aset_{B^\star}}\lambda(u)
P_h^{\phys}(\cdot\mid s,a\oplus_{B^\star}u)
\right)_{a\in\Aset}
:
\lambda\in\DeltaP(\Aset_{B^\star})
\right\}.
\label{eq:simultaneous-state-row-set}
\end{align}
Thus a single coefficient vector $\lambda$ determines all action rows at a
given $(h,s)$.  This common-mixture constraint records precisely that the
coalition chooses its overwrite without seeing the learner's realized action.

The corresponding family of Markov public-outcome kernels is
\begin{equation}
\begin{aligned}
\mathfrak M_{\rm sim}(B^\star)
&\coloneqq
\left\{
Q=(Q_h)_{h=1}^H:
\bigl(Q_h(\cdot\mid s,a)\bigr)_{a\in\Aset}
\in\Gamma^{B^\star,\rm sim}_{h,s}\right.\\[-2pt]
&\hspace{7.5em}\left.
\text{ for every }(h,s)
\right\}
\\[-2pt]
&=
\prod_{h\in[H],\,s\in\Sset}
\Gamma^{B^\star,\rm sim}_{h,s}.
\end{aligned}
\label{eq:simultaneous-s-rectangular-family}
\end{equation}
Equation~\eqref{eq:simultaneous-s-rectangular-family} is a stagewise
$s$-rectangular product: the row tuple chosen at one stage-state pair does
not restrict the tuple chosen at another stage-state pair.  It is generally
not $(s,a)$-rectangular, because the rows corresponding to different
planned actions at the same $(h,s)$ must use the same overwrite mixture
$\lambda$.  Rectangularity here refers to independent \emph{feasibility} of
the row-tuple choices; it does not assert probabilistic independence of the
resulting trajectory outcomes.

\begin{proposition}[Exact $s$-rectangular public-law equivalence]
\label{prop:simultaneous-s-rectangular-equivalence}
Fix $B^\star\subseteq[n]$.  Under the simultaneous-move information
structure, the conditional public-outcome row tuples attainable by
randomized, history-dependent, nonanticipating overwrite strategies are
exactly the nonanticipating selectors of
$\Gamma^{B^\star,\rm sim}$.  More precisely, at every public history
$\eta_h$ ending in $s$, such a selector chooses, before observing
$A_h$, a tuple
\begin{equation}
\boldsymbol\sigma_h(\eta_h)
=
\bigl(\sigma_{h,a}(\cdot\mid\eta_h)\bigr)_{a\in\Aset}
\in\Gamma^{B^\star,\rm sim}_{h,s},
\label{eq:simultaneous-history-selector}
\end{equation}
and the row indexed by the subsequently realized action $A_h=a$ is
$\sigma_{h,a}(\cdot\mid\eta_h)$.

Furthermore, for every $p\in\DeltaP(\Aset)$ and every bounded
Borel-measurable function $g:\Aset\times\mathcal Y\to\mathbb R$,
\begin{align}
&\inf_{\boldsymbol\nu=(\nu_a)_{a\in\Aset}
      \in\Gamma^{B^\star,\rm sim}_{h,s}}
\sum_{a\in\Aset}p(a)\E_{\nu_a}[g(a,R,S')]
\nonumber\\
&\qquad=
\min_{u\in\Aset_{B^\star}}
\sum_{a\in\Aset}p(a)
\E_{P_h^{\phys}(\cdot\mid s,a\oplus_{B^\star}u)}
[g(a,R,S')].
\label{eq:simultaneous-linear-row-minimization}
\end{align}
In particular, against any fixed current learner distribution $p$, a pure
overwrite is a worst-case simultaneous one-step response, although a mixed
overwrite may be needed as part of a saddle-point strategy.
\end{proposition}

\begin{proof}
We prove attainability in both directions and then establish the minimization
identity.

\emph{From overwrite strategies to row-tuple selectors.}
Fix a public history $\eta_h$ ending in state $s$.  Because the coalition
does not observe the current planned action, it must use one distribution
$q_h(\cdot\mid\eta_h)$ for all possible values of $A_h$.  Conditional on
$A_h=a$, the public-outcome law is therefore
\begin{equation}
Q_h^q(\cdot\mid\eta_h,a)
=
\sum_{u\in\Aset_{B^\star}}
q_h(u\mid\eta_h)
P_h^{\phys}(\cdot\mid s,a\oplus_{B^\star}u).
\label{eq:simultaneous-induced-action-row}
\end{equation}
The same coefficients $q_h(u\mid\eta_h)$ appear for every action $a$.
Consequently,
\[
\bigl(Q_h^q(\cdot\mid\eta_h,a)\bigr)_{a\in\Aset}
\in\Gamma^{B^\star,\rm sim}_{h,s}.
\]
Measurability in $\eta_h$ follows from measurability of the overwrite
kernel and finiteness of $\Aset_{B^\star}$.  Hence every feasible
simultaneous overwrite strategy induces a nonanticipating selector of the
statewise row-tuple family.

\emph{From row-tuple selectors to overwrite strategies.}
Write
\[
\Aset_{B^\star}=\{u_1,\ldots,u_m\},
\]
and define
\begin{align}
T_{h,s}:\DeltaP(\Aset_{B^\star})
&\longrightarrow\Gamma^{B^\star,\rm sim}_{h,s},
\nonumber\\
T_{h,s}(\lambda)
&\coloneqq
\sum_{j=1}^m\lambda_j\mathbf P^{u_j}_{h,s}.
\label{eq:simultaneous-row-tuple-map}
\end{align}
By the definition of the convex hull, $T_{h,s}$ is surjective.  Under the
weak topology on $\DeltaP(\mathcal Y)^{\Aset}$, it is a continuous map
between compact metric spaces.  Applying a standard Borel measurable-selection
theorem to the nonempty compact-valued correspondence
$\boldsymbol\nu\mapsto T_{h,s}^{-1}(\boldsymbol\nu)$ yields a Borel right
inverse of $T_{h,s}$.  Apply this right inverse to the selector in
\eqref{eq:simultaneous-history-selector} and use the resulting coefficient
vector as $q_h(\cdot\mid\eta_h)$.  Equation
\eqref{eq:simultaneous-induced-action-row} then agrees with the selected row
for every possible action.  Repeated composition of the behavioral kernels
and the physical outcome kernels shows that the overwrite strategy and the
row-tuple selector induce the same law of the complete public trajectory under
every learner policy.

\emph{Linear minimization.}
Every $\boldsymbol\nu\in\Gamma^{B^\star,\rm sim}_{h,s}$ can be written as
$\boldsymbol\nu=\sum_u\lambda(u)\mathbf P^u_{h,s}$ for some
$\lambda\in\DeltaP(\Aset_{B^\star})$.  Therefore,
\begin{align}
\sum_a p(a)\E_{\nu_a}[g(a,R,S')]
&=
\sum_u\lambda(u)
\sum_a p(a)
\E_{P_h^{\phys}(\cdot\mid s,a\oplus_{B^\star}u)}
[g(a,R,S')]
\nonumber\\
&\ge
\min_u
\sum_a p(a)
\E_{P_h^{\phys}(\cdot\mid s,a\oplus_{B^\star}u)}
[g(a,R,S')].
\label{eq:simultaneous-mixture-lower-bound}
\end{align}
Conversely, every pure tuple $\mathbf P^u_{h,s}$ belongs to the convex
hull.  Selecting a minimizing $u$, which exists because the overwrite
action set is finite, gives the reverse inequality.  This proves
\eqref{eq:simultaneous-linear-row-minimization}.
\end{proof}

\textit{Bellman evaluation of a fixed Markov policy.}
Let $\Pi_{\mathrm{M,stoch}}$ denote the randomized nonstationary Markov
learner policies.  For $\pi\in\Pi_{\mathrm{M,stoch}}$, define
\begin{align}
V_{H+1}^{\pi,\rm sim}(s)&=0,
\nonumber\\
V_h^{\pi,\rm sim}(s)
&=
\inf_{\boldsymbol\nu=(\nu_a)_a
      \in\Gamma^{B^\star,\rm sim}_{h,s}}
\sum_{a\in\Aset}\pi_h(a\mid s)
\E_{\nu_a}\!\left[R+V_{h+1}^{\pi,\rm sim}(S')\right]
\nonumber\\
&=
\min_{u\in\Aset_{B^\star}}
\sum_{a\in\Aset}\pi_h(a\mid s)
\E_{P_h^{\phys}(\cdot\mid s,a\oplus_{B^\star}u)}
\!\left[R+V_{h+1}^{\pi,\rm sim}(S')\right].
\label{eq:simultaneous-policy-bellman}
\end{align}
The second equality follows from
\eqref{eq:simultaneous-linear-row-minimization}.  We now verify that this
recursion equals the operational value against the entire full-history
response class.  For any $q\in\Omega_{\rm sim}(B^\star)$ and any public
history $\eta_h$, let
\begin{equation}
J_h^{\pi,q}(\eta_h)
\coloneqq
\E_{\eta_h}^{\pi,q}\!\left[\sum_{t=h}^H R_t\right].
\label{eq:simultaneous-fixed-policy-continuation}
\end{equation}
We claim that, for every $q$ and every $\eta_h$ ending in state $s$,
\begin{equation}
J_h^{\pi,q}(\eta_h)\ge V_h^{\pi,\rm sim}(s),
\label{eq:simultaneous-fixed-policy-lower-bound}
\end{equation}
and that equality is attained by one deterministic nonstationary Markov
overwrite strategy.

At $h=H+1$, both sides of
\eqref{eq:simultaneous-fixed-policy-lower-bound} are zero.  Suppose the claim
holds at stage $h+1$, and let $\eta_h$ end in $s$.  Conditioning on the
two simultaneous action draws and then on the public outcome yields
\begin{align}
J_h^{\pi,q}(\eta_h)
&=
\sum_{a,u}\pi_h(a\mid s)q_h(u\mid\eta_h)
\E_{P_h^{\phys}(\cdot\mid s,a\oplus_{B^\star}u)}
\!\left[R+J_{h+1}^{\pi,q}(\eta_h,a,R,S')\right]
\nonumber\\
&\ge
\sum_{u}q_h(u\mid\eta_h)
\sum_a\pi_h(a\mid s)
\E_{P_h^{\phys}(\cdot\mid s,a\oplus_{B^\star}u)}
\!\left[R+V_{h+1}^{\pi,\rm sim}(S')\right]
\nonumber\\
&\ge
\min_{u}
\sum_a\pi_h(a\mid s)
\E_{P_h^{\phys}(\cdot\mid s,a\oplus_{B^\star}u)}
\!\left[R+V_{h+1}^{\pi,\rm sim}(S')\right]
\nonumber\\
&=V_h^{\pi,\rm sim}(s).
\label{eq:simultaneous-fixed-policy-induction-step}
\end{align}
The first inequality applies the induction hypothesis after every possible
next public history; the second uses that an average is at least its smallest
term.  For every $(h,s)$, choose
\begin{equation}
u_h^\pi(s)
\in
\operatorname*{arg\,min}_{u\in\Aset_{B^\star}}
\sum_a\pi_h(a\mid s)
\E_{P_h^{\phys}(\cdot\mid s,a\oplus_{B^\star}u)}
\!\left[R+V_{h+1}^{\pi,\rm sim}(S')\right].
\label{eq:simultaneous-deterministic-best-response}
\end{equation}
If the coalition uses $U_h=u_h^\pi(S_h)$, then backward induction makes
both inequalities in
\eqref{eq:simultaneous-fixed-policy-induction-step} equalities.  Consequently,
\begin{equation}
W_{\rm sim}(\pi)=V_1^{\pi,\rm sim}(s_1).
\label{eq:simultaneous-fixed-policy-value}
\end{equation}
Thus every fixed
stochastic Markov learner policy has a deterministic nonstationary Markov
worst-case response $u_h^\pi(s)$.  Unlike the observed-action setting, this
response cannot depend on the current realized planned action.

\begin{theorem}[Stochastic Markov sufficiency under simultaneous moves]
\label{thm:simultaneous-markov-sufficiency}
Suppose the learner and coalition move simultaneously at every stage and
there are no cross-stage overwrite budgets or other coupling constraints.
Define
\begin{align}
V_{H+1}^{\star,\rm sim}(s)&=0,
\nonumber\\
V_h^{\star,\rm sim}(s)
&=
\max_{p\in\DeltaP(\Aset)}
\min_{\lambda\in\DeltaP(\Aset_{B^\star})}
\sum_{a\in\Aset}
\sum_{u\in\Aset_{B^\star}}
p(a)\lambda(u)
G_h^\star(s,a,u),
\label{eq:simultaneous-optimal-bellman}
\end{align}
where
\begin{equation}
G_h^\star(s,a,u)
\coloneqq
\E_{P_h^{\phys}(\cdot\mid s,a\oplus_{B^\star}u)}
\!\left[R+V_{h+1}^{\star,\rm sim}(S')\right].
\label{eq:simultaneous-bellman-matrix}
\end{equation}
Then
\begin{equation}
v_{\rm sim}^\star
=
\max_{\pi\in\Pi_{\mathrm{M,stoch}}}W_{\rm sim}(\pi)
=V_1^{\star,\rm sim}(s_1).
\label{eq:simultaneous-policy-class-equivalence}
\end{equation}
Moreover, the simultaneous security game admits a saddle pair consisting of
a stochastic nonstationary Markov learner policy and a stochastic
nonstationary Markov overwrite policy.  In general, neither member of this
saddle pair can be required to be deterministic.
\end{theorem}

\begin{proof}
For each $(h,s)$, the expression in
\eqref{eq:simultaneous-optimal-bellman} is the value of a finite zero-sum
matrix game with payoff matrix
$(G_h^\star(s,a,u))_{a,u}$.  Von Neumann's minimax theorem gives
\begin{align}
V_h^{\star,\rm sim}(s)
&=
\max_{p\in\DeltaP(\Aset)}
\min_{\lambda\in\DeltaP(\Aset_{B^\star})}
\sum_{a,u}p(a)\lambda(u)G_h^\star(s,a,u)
\nonumber\\
&=
\min_{\lambda\in\DeltaP(\Aset_{B^\star})}
\max_{p\in\DeltaP(\Aset)}
\sum_{a,u}p(a)\lambda(u)G_h^\star(s,a,u).
\label{eq:simultaneous-stage-minimax}
\end{align}
Because both simplices are compact, there exist saddle distributions
$p_h^\star(\cdot\mid s)$ and
$\lambda_h^\star(\cdot\mid s)$.  Let $\pi^\star$ and $q^\star$ be
the corresponding stochastic nonstationary Markov policies.

We first prove that $\pi^\star$ guarantees
$V_h^{\star,\rm sim}(s)$ against every randomized history-dependent
response.  Fix any $q\in\Omega_{\rm sim}(B^\star)$.  We claim that after
every public history $\eta_h$ ending in $s$,
\begin{equation}
\E_{\eta_h}^{\pi^\star,q}
\!\left[\sum_{t=h}^H R_t\right]
\ge V_h^{\star,\rm sim}(s).
\label{eq:simultaneous-maximin-induction}
\end{equation}
The claim is immediate at $h=H+1$.  If it holds at $h+1$, then
conditioning on the simultaneous action draws and using the induction
hypothesis after every possible next public history gives
\begin{align}
\E_{\eta_h}^{\pi^\star,q}
\!\left[\sum_{t=h}^H R_t\right]
&\ge
\sum_{a,u}
p_h^\star(a\mid s)q_h(u\mid\eta_h)
G_h^\star(s,a,u)
\nonumber\\
&\ge
\min_{\lambda\in\DeltaP(\Aset_{B^\star})}
\sum_{a,u}p_h^\star(a\mid s)\lambda(u)
G_h^\star(s,a,u)
\nonumber\\
&=V_h^{\star,\rm sim}(s).
\label{eq:simultaneous-maximin-step}
\end{align}
Thus $W_{\rm sim}(\pi^\star)\ge V_1^{\star,\rm sim}(s_1)$.

We next prove that the Markov response $q^\star$ holds every randomized
history-dependent learner policy below the same value.  Fix
$\varpi\in\Pi_{\mathrm{hist}}$.  We claim that after every public history
$\eta_h$ ending in $s$,
\begin{equation}
\E_{\eta_h}^{\varpi,q^\star}
\!\left[\sum_{t=h}^H R_t\right]
\le V_h^{\star,\rm sim}(s).
\label{eq:simultaneous-minimax-induction}
\end{equation}
Again, the claim is immediate at $H+1$.  Assuming it at $h+1$ and
writing $p(a)=\varpi_h(a\mid\eta_h)$, we obtain
\begin{align}
\E_{\eta_h}^{\varpi,q^\star}
\!\left[\sum_{t=h}^H R_t\right]
&\le
\sum_{a,u}p(a)\lambda_h^\star(u\mid s)
G_h^\star(s,a,u)
\nonumber\\
&\le
\max_{\widetilde p\in\DeltaP(\Aset)}
\sum_{a,u}\widetilde p(a)\lambda_h^\star(u\mid s)
G_h^\star(s,a,u)
\nonumber\\
&=V_h^{\star,\rm sim}(s).
\label{eq:simultaneous-minimax-step}
\end{align}
At the initial state, these two induction bounds imply
\begin{align*}
W_{\rm sim}(\pi^\star)
&\ge V_1^{\star,\rm sim}(s_1),
\\
W_{\rm sim}(\varpi)
&\le
\E^{\varpi,q^\star}\!\left[\sum_{h=1}^H R_h\right]
\le V_1^{\star,\rm sim}(s_1)
\qquad
\text{for every }\varpi\in\Pi_{\mathrm{hist}}.
\end{align*}
Taking the supremum over $\varpi$, and observing that
$\pi^\star\in\Pi_{\mathrm{M,stoch}}\subseteq\Pi_{\mathrm{hist}}$, proves
\eqref{eq:simultaneous-policy-class-equivalence}.  The same inequalities give
the saddle property of $(\pi^\star,q^\star)$.
\end{proof}

\begin{remark}[Why randomization may be essential]
\label{rem:simultaneous-mixing-essential}
The simultaneous-move result establishes sufficiency of stochastic Markov
policies, not deterministic Markov policies.  Consider a one-stage problem
with one honest binary action $x\in\{0,1\}$, one overwritten binary action
$u\in\{0,1\}$, and deterministic reward
$R=\mathbf 1\{x\ne u\}$.  A learner that chooses $x$ uniformly guarantees
reward $1/2$, whereas every deterministic learner action can be matched by
the coalition and receives reward zero.  Conversely, a uniform overwrite is
needed to hold every learner action distribution below $1/2$.  Hence the
simultaneous game has value $1/2$, but neither player has a deterministic
saddle strategy.  This differs from the observed-planned-action setting,
where the coalition moves after seeing the learner's realized action and a
deterministic learner policy is sufficient.
\end{remark}

\section{Proofs for Section \ref{sec:hardness}}

\begin{theorem}[Restatement of Theorem~\ref{thm:security-impossibility}] 
For every possibly randomized online algorithm and every $K\ge1$, there are
two horizon-one, one-state Byzantine team MDPs with two binary-action agents,
the same fixed and even known identity $B^\star=\{2\}$, and the same
deterministic realized overwrite strategy, such that their public transcript
laws are identical and
\begin{equation}
  \Reg_K=0
  \quad\text{almost surely in both environments},
\label{eq:zero-return-regret-hardness}
\end{equation}
while
\begin{equation}
  \E_0\Regsec_K+\E_1\Regsec_K=K.
\label{eq:security-sum-lower-bound}
\end{equation}
Consequently, one environment satisfies
\begin{equation}
  \E_\theta\Regsec_K\ge K/2,
  \qquad
  \E_\theta D_K=\E_\theta\Regsec_K\ge K/2.
\label{eq:linear-security-lower-bound}
\end{equation}
No algorithm can therefore guarantee $o(K)$ security regret uniformly over
arbitrary feasible realized responses from public trajectory feedback alone.
\end{theorem}
\begin{proof} 
Fix an arbitrary, possibly randomized online learning algorithm and an integer
$K\ge1$.  We construct two environments
$\mathcal M_0$ and $\mathcal M_1$, indexed by
$\theta\in\{0,1\}$.

\textit{Step 1: common state, action, and attacker structure.}
Both environments have horizon $H=1$, the singleton public state space
$\Sset=\{s\}$, and two agents with action spaces
\[
\Aset_1=\Aset_2=\{0,1\}.
\]
The planned joint action is denoted by $a=(x,y)$, where
$x\in\Aset_1$ and $y\in\Aset_2$.  In both environments, the Byzantine
identity is
\[
B^\star=\{2\},
\]
and this identity may be disclosed to the learner.  If the coalition chooses
an overwrite $u\in\Aset_{B^\star}=\{0,1\}$, the executed action is
\[
a\oplus_{B^\star}u=(x,u).
\]
Thus the learner's planned second coordinate $y$ is always replaced and has
no effect on the physical outcome.

The realized responder is the same deterministic, nonanticipating strategy in
both environments:
\begin{equation}
\bar q_1(u\mid s,a)=\one\{u=0\}
\qquad
\text{for every planned action }a.
\label{eq:appendix-benign-response}
\end{equation}
This strategy is feasible and observes no future information.

\textit{Step 2: the two physical kernels.}
For $\theta\in\{0,1\}$, define the deterministic reward associated with an
executed action $(x,u)$ by
\begin{equation}
r_\theta(x,u)
\coloneqq
\one\{u=0\ \text{or}\ x=\theta\}.
\label{eq:appendix-hard-reward}
\end{equation}
The physical public-outcome kernel in environment $\mathcal M_\theta$ is the
point mass
\begin{equation}
P_{1,\theta}^{\phys}
\bigl((R,S')=(r_\theta(x,u),s)\mid s,(x,u)\bigr)=1.
\label{eq:appendix-hard-physical-kernel}
\end{equation}
The reward lies in $[0,1]$, as required.  The next state is immaterial
because the horizon is one.

\textit{Step 3: exact equality of public transcript laws.}
Under the realized responder $\bar q$, one has $U_{k,1}=0$ in every
episode.  Therefore, for both $\theta=0$ and $\theta=1$,
\begin{equation}
R_{k,1}=r_\theta(X_{k,1},0)=1
\qquad\text{almost surely},
\label{eq:appendix-observed-reward-one}
\end{equation}
regardless of the learner's planned action.

For completeness, couple the two experiments using the same internal random
seed for the online algorithm.  We prove inductively that their complete public
transcripts are equal.  Before episode one, the transcripts are empty.  If the
transcripts through episode $k-1$ are equal, the algorithm receives the same
input and, under the common random seed, selects the same episode-
$k$ behavioral policy and the same planned action in the two experiments.
The responder then selects $u=0$ in both experiments, and
\eqref{eq:appendix-observed-reward-one} gives the same reward and next state.
Thus the transcripts through episode $k$ are equal.  Induction over
$k=1,\ldots,K$ proves pathwise equality under the coupling, and hence the
two public transcript laws are identical.

In particular, if $\varpi_k\in\DeltaP(\Aset_1\times\Aset_2)$ denotes the
possibly random behavioral policy deployed in episode $k$, then the same
$\varpi_k$ is deployed in both coupled experiments.  Define its marginal
probability of choosing first coordinate $x=\theta$ by
\begin{equation}
p_{k,\theta}
\coloneqq
\sum_{y\in\{0,1\}}\varpi_k((\theta,y)\mid s).
\label{eq:appendix-policy-marginal}
\end{equation}
For every realization of $\varpi_k$,
\begin{equation}
p_{k,0}+p_{k,1}=1.
\label{eq:appendix-marginals-sum-one}
\end{equation}

\textit{Step 4: operational security value of an arbitrary policy.}
Fix $\theta$ and an arbitrary one-stage policy
$\varpi\in\DeltaP(\Aset_1\times\Aset_2)$.  Because the coalition observes
the realized planned action before overwriting, the smallest reward attainable
after a plan $(x,y)$ is
\begin{align}
\min_{u\in\{0,1\}}r_\theta(x,u)
&=
\min_{u\in\{0,1\}}\one\{u=0\ \text{or}\ x=\theta\}
\nonumber\\
&=\one\{x=\theta\}.
\label{eq:appendix-pointwise-worst-reward}
\end{align}
Indeed, if $x=\theta$, both overwrites yield reward one; if
$x\ne\theta$, the pure overwrite $u=1$ yields reward zero.  Randomizing
the overwrite cannot produce a smaller expectation than this pointwise
minimum.  It follows that
\begin{align}
W_\theta(\varpi)
&=
\sum_{x,y}\varpi((x,y)\mid s)
\min_{u\in\{0,1\}}r_\theta(x,u)
\nonumber\\
&=
\sum_{y\in\{0,1\}}\varpi((\theta,y)\mid s).
\label{eq:appendix-policy-security-value}
\end{align}
Choosing the deterministic policy whose first coordinate is always
$x=\theta$ gives security value one.  Since the stage reward is at most one,
we conclude that
\begin{equation}
v_\theta^\star=1.
\label{eq:appendix-optimal-security-value}
\end{equation}
Applying \eqref{eq:appendix-policy-security-value} to the deployed policy
$\varpi_k$ gives
\begin{equation}
W_\theta(\varpi_k)=p_{k,\theta}.
\label{eq:appendix-deployed-security-value}
\end{equation}

\textit{Step 5: return regret is zero.}
By \eqref{eq:appendix-observed-reward-one}, the conditional expected factual
return in every episode is
\begin{equation}
m_{k,\theta}
=
\E_\theta[R_{k,1}\mid
\F_{k-1}\vee\sigma(\varpi_k)]
=1.
\label{eq:appendix-factual-return}
\end{equation}
Combining this identity with
\eqref{eq:appendix-optimal-security-value} yields, almost surely in each
environment,
\begin{equation}
\Reg_{K,\theta}
=
\sum_{k=1}^K(v_\theta^\star-m_{k,\theta})
=0.
\label{eq:appendix-zero-return-regret}
\end{equation}

\textit{Step 6: the security regrets add to $K$.}
Under the coupling constructed above, the episode-$k$ security-regret gaps
in the two environments satisfy, pathwise,
\begin{align}
&\bigl(v_0^\star-W_0(\varpi_k)\bigr)
+\bigl(v_1^\star-W_1(\varpi_k)\bigr)
\nonumber\\
&\qquad=
(1-p_{k,0})+(1-p_{k,1})
=1,
\label{eq:appendix-one-step-security-sum}
\end{align}
where the last equality uses
\eqref{eq:appendix-marginals-sum-one}.  Summing over episodes gives the stronger
coupling identity
\begin{equation}
\Reg^{\rm sec}_{K,0}+\Reg^{\rm sec}_{K,1}=K
\label{eq:appendix-pathwise-security-sum}
\end{equation}
almost surely under the coupling.
Taking expectations under the two marginal experiments proves
\begin{equation}
\E_0\Reg^{\rm sec}_{K,0}
+\E_1\Reg^{\rm sec}_{K,1}
=K.
\label{eq:appendix-expected-security-sum}
\end{equation}
Therefore at least one $\theta\in\{0,1\}$ satisfies
\begin{equation}
\E_\theta\Reg^{\rm sec}_{K,\theta}\ge \frac K2.
\label{eq:appendix-one-environment-linear-security}
\end{equation}

Finally, in either environment,
\eqref{eq:appendix-factual-return} and \eqref{eq:appendix-deployed-security-value} imply
\begin{align}
D_{K,\theta}
&=
\sum_{k=1}^K
\bigl(m_{k,\theta}-W_\theta(\varpi_k)\bigr)
\nonumber\\
&=
\sum_{k=1}^K(1-p_{k,\theta})
=\Reg^{\rm sec}_{K,\theta}.
\label{eq:appendix-response-gap-equals-security-regret}
\end{align}
Thus the same environment selected in
\eqref{eq:appendix-one-environment-linear-security} also satisfies
$\E_\theta D_{K,\theta}\ge K/2$.  If an algorithm guaranteed
$o(K)$ expected security regret uniformly over all feasible realized
responses and environments, both terms on the left side of
\eqref{eq:appendix-expected-security-sum} would be $o(K)$, contradicting their
exact sum $K$.  This completes the proof.
\end{proof}

\section{Proofs for Section \ref{sec:learning}}
\label{app:learning-proofs}

This appendix proves Theorems~\ref{thm:stage-tied-estimation} and
\ref{thm:stage-tied-regret}.  The argument builds on four components of
\cite{appel2025regret}: their partial-halfspace surrogate,
directed-Hellinger wealth-based estimator, robust DEC bound, and E2D master
inequality.  We cite a reusable result at the point where it is invoked.  We
rederive every step changed by our private coding channel, original payoff
scale, shared layer-level certificate, full terminal outcome, or approximate
measurable selection.  The rate-critical new statistical ingredient is the
layer-level calibration argument: it charges one rounded certificate per
stage rather than one per predecessor state.

\subsection{Original-scale coding and the Appel--Kosoy adapter}
\label{app:learning-original-scale}

Fix the unknown but episode-invariant identity \(B^\star\).  For this appendix,
write
\[
\Gamma_{h,s,a}\coloneqq\Gamma_{h,s,a}^{B^\star}
\subseteq\Delta([0,1]\times\mathcal S)
\]
for the public ambiguity row obtained from the fixed-identity reduction.  Each
row is nonempty, compact, and convex.  The actual randomized history-dependent
responder selects a conditional public-outcome law from this row after observing
the current planned action.

Adapting the Bernoulli reward conversion in Appendix~H of
\cite{appel2025regret}, we use success probability \(R/H\) and score a
success as \(HY\).  The resulting expectation remains on the original task
scale.  The deferred learner-private sampling and filtration argument needed
for a full-history response are specific to our setting and are proved in
Lemma~\ref{lem:app-deferred-coding}.

Define the private coding kernel \(C_H\) by
\begin{align}
C_H((1,s')\mid r,s')&=r/H,
\nonumber\\
C_H((0,s')\mid r,s')&=1-r/H,
\qquad r\in[0,1].
\label{eq:app-bernoulli-kernel}
\end{align}
For a law \(p\), let \(C_{H\#}p\) be its image through this kernel.  The
coded ambiguity row is
\begin{equation}
\widetilde\Gamma_{h,s,a}
\coloneqq
\left\{C_{H\#}p:p\in\Gamma_{h,s,a}\right\}
\subseteq\Delta(\mathcal O).
\label{eq:app-converted-row}
\end{equation}

\begin{lemma}[Deferred private coding and filtrations]
\label{lem:app-deferred-coding}
Enlarge the probability space by independent variables
\(U_{k,h}\sim\operatorname{Unif}[0,1]\), independent of the players' random
seeds and of the physical process conditional on those seeds, and set
\begin{equation}
Y_{k,h}=\mathbf 1\{U_{k,h}\le R_{k,h}/H\}.
\label{eq:app-deferred-bit-realization}
\end{equation}
The experiment in which all \(U_{k,h}\) are sampled after episode \(k\) has
the same joint law of public trajectories, coded outcomes, and future learner
decisions as the experiment in which \(U_{k,h}\) is sampled immediately after
\(R_{k,h}\), kept hidden from both players during the episode, and revealed to
the learner only after that episode.  In the latter representation,
conditioning the analysis on any coded past still leaves the next coded
post-action law in \(\widetilde\Gamma_{h,s,a}\).

More precisely, let \(\mathcal G_{k,h}^{\rm adv}\) contain the public
pre-action history, the coalition's random seed, and any episode-start side
information permitted by the theorem, but not the auxiliary variables \(U\)
or the private bits.  Let \(\mathcal F_k^{\rm L}\) (denoted
\(\mathcal F_k\) in Section~\ref{sec:learning}) be the learner's filtration,
which also contains the coded bits after the end of each episode.  A response
may be an arbitrary randomized full-history kernel measurable with respect to
\(\mathcal G_{k,h}^{\rm adv}\) and the realized planned action.  It need not be
Markov and may use all past public rewards.  Enlarging an \emph{analysis}
filtration to contain learner-private variables is only a conditioning device
and does not make those variables available to the response; conditional laws
under such an enlargement are still barycenters of feasible row laws.
\end{lemma}

\begin{proof}
Conditional on the physical trajectory, equation
\eqref{eq:app-deferred-bit-realization} is a product of Bernoulli kernels whose
parameters are fixed by the observed rewards.  Sampling the independent
uniform variables sequentially or deferring all of them therefore gives the
same conditional law.  The bits are unavailable until the episode ends in
both implementations, so neither the within-episode learner actions nor the
response kernels change; after their common reveal, every future learner
decision has the same conditional distribution.  This proves equality of the
full joint laws by induction over episodes.

For the final assertion, fix a coded history ending in \((h,s,a)\).  Conditional
on each compatible physical public history, the response selects some law
\(p\in\Gamma_{h,s,a}\), and the conditional coded law is
\(C_{H\#}p\in\widetilde\Gamma_{h,s,a}\).  Conditioning only on the coded
history averages these laws over the compatible physical histories.  Because
\(\widetilde\Gamma_{h,s,a}\) is convex and closed, that barycenter remains in
the same row.  This argument permits arbitrary public-history dependence of
the response and imposes no Markov restriction.
\end{proof}

\begin{lemma}[Private coding preserves Markov values in task units]
\label{lem:app-conversion}
The set in equation~\eqref{eq:app-converted-row} is nonempty, compact, and
convex.  Every actual nonanticipating public-law selector induces, after the
private Bernoulli draw, a nonanticipating selector of
\(\widetilde\Gamma\).  For every function
\(v:\mathcal S\to\mathbb R\) and every \(p\in\Gamma_{h,s,a}\),
\begin{equation}
\mathbb E_{C_{H\#}p}
[HY+v(S')]
=
\mathbb E_p[R+v(S')].
\label{eq:app-conversion-expectation}
\end{equation}
Consequently, when the coded reward is scored as \(HY\), every
nonstationary Markov policy has exactly the same robust and factual value as
in the physical model, and the optimal security values are equal.
\end{lemma}

\begin{proof}
The map \(p\mapsto C_{H\#}p\) is affine and continuous in the
finite-dimensional weak topology.  The continuous image of a nonempty compact
set is nonempty and compact, and the affine image of a convex set is convex.
This proves the first assertion.

Lemma~\ref{lem:app-deferred-coding} shows that the coded conditional laws form
a nonanticipating selector of \(\widetilde\Gamma\), including when the original
response depends on the complete public reward history.

To prove equation~\eqref{eq:app-conversion-expectation}, condition first on
\((R,S')=(r,s')\).  By equation~\eqref{eq:app-bernoulli-kernel},
\[
\mathbb E[HY+v(S')\mid R=r,S'=s']
=r+v(s').
\]
Integrating this identity with respect to \(p\) gives
equation~\eqref{eq:app-conversion-expectation}.  Applying this row identity
backward shows task-scale Bellman-value equality for every Markov policy.
Maximizing the same recursions gives equality of optimal values, and the
exact selector reduction gives operational equality.  We deliberately do
not claim policy-by-policy equality for arbitrary reward-history-dependent
team policies: coarsening a real reward history can change such a policy.
This restriction is harmless because Theorem~\ref{thm:markov-sufficiency}
shows Markov policies suffice for the benchmark and Section~\ref{sec:learning}
deploys only Markov policies.
\end{proof}

\begin{lemma}[Original-scale Appel--Kosoy surrogate]
\label{lem:app-surrogate-applicability}
The coded model, evaluated with task-scale reward \(HY\), admits the
surrogate \(N^\circ\) in equation~\eqref{eq:learning-stage-class}.  It
contains every coded physical row at the recommended actions, contains
every coded feasible selector, is \(H\)-bounded, and has task-scale optimal
value \(v^\star\).
\end{lemma}

\begin{proof}
The coded state, action, and outcome spaces are finite.  By
Lemma~\ref{lem:app-conversion}, its rows are nonempty, compact, convex, and
state--action rectangular.  To invoke the 1-bounded statement of
\cite{appel2025regret}, use \(Y\) as its reward coordinate.  Every coded
physical row has expected \(Y\le1/H\).  Choose one such row at every
\((h,s,a)\) to obtain a fixed feasible Markov selector \(\sigma^\circ\).  Its
conditional expected suffix reward is at most
\[
\sum_{t=h}^{H}\frac1H=\frac{H-h+1}{H}\le1.
\]
The surrogate construction used in the proof of their Corollary~3 therefore
applies in the unit certificate coordinates
\[
\bar V_h^\star\coloneqq V_h^\star/H.
\]
At a recommended action, its partial halfspace is
\[
\left\{\nu:
\mathbb E_\nu[Y+\bar V_{h+1}^\star(S')]
\ge\bar V_h^\star(s)\right\}
=
\Psi_H(V_{h+1}^\star,V_h^\star(s)).
\]
Thus the set-valued surrogate is exactly \(N^\circ\).  Their construction
gives row and selector containment and equality of optimal values in unit
coordinates.  In particular, row containment makes \(\sigma^\circ\) feasible
for \(N^\circ\), and its conditional expected task-scale suffix return is at
most \(H-h+1\le H\).  Hence the surrogate is \(H\)-bounded in the sense of
Section~\ref{sec:learning-target}.
Multiplying its reward functional and value by \(H\) gives
task-scale value \(v^\star\), with no change to the outcome laws or
halfspaces.

At stage \(H\), \cite{appel2025regret} identify all physical next states
with one terminal state.  Define
\(T_{\rm term}(y,s')=(y,s_\dagger)\).  For the particular surrogate constructed
above, \(\bar V_{H+1}^\star\equiv0\), so its recommended terminal row is the
reward-only halfspace
\(C_c=\{p:\mathbb E_p[Y]\ge c\}\).  Consequently,
\[
T_{\rm term\#}C_c
=\left\{q:\mathbb E_q[Y]\ge c\right\},
\qquad
C_c=\{p:T_{\rm term\#}p\in T_{\rm term\#}C_c\}.
\]
Every member of the coarsened halfspace can be lifted by attaching any fixed
next state.  Thus, for this value-preserving surrogate, coarsening and lifting
preserve row containment, value equality, and \(1\)-boundedness.  A generic
stage-tied terminal row need not be the inverse image of its coarsening; the
DEC comparison below instead pushes each generic row forward and uses
Hellinger data processing.  Because our auxiliary reward \(Y\) is already
binary, the additional Bernoulli conversion in their modified loss is the
identity.  This verifies the imported surrogate hypotheses; the coding,
task-scale conversion, and terminal-state comparison are the adapters proved
here.
\end{proof}

\subsection{Imported Hellinger facts and certificate rounding}
\label{app:learning-hellinger}

For unit certificate coordinates \(\bar v\in[0,1]^{\mathcal S}\) and
\(\bar c\in[0,1]\), abbreviate
\[
\Psi(\bar v,\bar c)
\coloneqq
\left\{\nu:\mathbb E_\nu[Y+\bar v(S')]\ge\bar c\right\}
=\Psi_H(H\bar v,H\bar c).
\]

For a nonempty compact convex set \(C\subseteq\Delta(\mathcal O)\) and a
full-support law \(p\), let
\[
D_{\rm H}^2(p,q)
\coloneqq1-\sum_{o\in\mathcal O}\sqrt{p(o)q(o)},
\qquad
q_C(p)\in\operatorname*{arg\,min}_{q\in C}D_{\rm H}^2(p,q),
\quad
d_C(p)=D_{\rm H}^2(p,q_C(p)).
\]

\begin{lemma}[Hellinger projection and rounding facts]
\label{lem:app-imported-hellinger}
The following facts hold.
\begin{enumerate}[label=\textup{(\roman*)},leftmargin=2em]
\item If \(p\) has full support, then \(q_C(p)\) is unique and varies
continuously with \(p\) on the full-support simplex.
\item With \(q=q_C(p)\), \(d=d_C(p)\), and
\(b(o)=\sqrt{q(o)/p(o)}+d\),
\begin{equation}
\mathbb E_p b=1,
\qquad
\mathbb E_\nu[b^{-1}]\le1-d
\quad\text{for every }\nu\in C.
\label{eq:app-projection-bet-properties}
\end{equation}
\item The map \(p\mapsto d_C(p)\) is convex on the simplex.  At every
full-support point it is differentiable along feasible simplex directions,
and its one-sided directional derivative satisfies
\begin{equation}
D\{2d_C\}(p)[\nu-p]
=-\sum_{o\in\mathcal O}
\sqrt{\frac{q(o)}{p(o)}}\{\nu(o)-p(o)\},
\label{eq:app-directional-derivative}
\end{equation}
and hence
\begin{equation}
\mathbb E_{o\sim\nu}[b(o)-1]
=-D\{2d_C\}(p)[\nu-p].
\label{eq:app-bet-derivative}
\end{equation}
\item If \(v^\uparrow\ge v\) coordinatewise,
\(\|v^\uparrow-v\|_\infty\le\varepsilon_v\),
\(c^\downarrow\le c\), and
\(c-c^\downarrow\le\varepsilon_c\), then
\begin{equation}
\Psi(v,c)\subseteq\Psi(v^\uparrow,c^\downarrow)
\label{eq:app-rounded-inclusion}
\end{equation}
and, for every \(p\in\Delta(\mathcal O)\),
\begin{equation}
d_{\Psi(v,c)}(p)
\le
2d_{\Psi(v^\uparrow,c^\downarrow)}(p)
+2(\varepsilon_v+\varepsilon_c).
\label{eq:app-rounded-loss}
\end{equation}
\end{enumerate}
\end{lemma}

\begin{proof}
Items (i) and (iii) specialize
\cite[Lemmas~1, 2, and~4]{appel2025regret} to the directed squared-Hellinger
convention.  Under that convention,
their point-to-set loss is exactly \(d_C(p)\); no factor of two or reversal of
the directed loss is introduced.

For item (ii), the affinity identity gives
\(\mathbb E_pb=\sum_o\sqrt{p(o)q(o)}+d=1\).  The projection first-order
condition requires a short boundary qualification.  If \(q(o_0)=0\) and
some \(\nu\in C\) had \(\nu(o_0)>0\), then
\(q_t=(1-t)q+t\nu\in C\).  At coordinate \(o_0\), its affinity with the
full-support law \(p\) gains
\(\sqrt{t\,p(o_0)\nu(o_0)}\); changes on coordinates where \(q\) is positive
are \(O(t)\), and gains at its other zero coordinates are nonnegative.
Thus \(q_t\) would have strictly larger affinity for all sufficiently small
\(t>0\), contradicting the definition of \(q\).  Hence
\begin{equation}
q(o)=0\quad\Longrightarrow\quad \nu(o)=0
\quad\text{for every }\nu\in C.
\label{eq:app-projection-common-support}
\end{equation}
On the support of \(q\), differentiating the affinity in every feasible
direction \(\nu-q\) now gives
\[
\sum_{o:q(o)>0}\nu(o)\sqrt{p(o)/q(o)}
\le
\sum_{o:q(o)>0}\sqrt{p(o)q(o)}=1-d.
\]
Since \(b(o)^{-1}\le\sqrt{p(o)/q(o)}\) on this support and
equation~\eqref{eq:app-projection-common-support} removes all remaining
coordinates from \(\mathbb E_\nu[b^{-1}]\), the second inequality in
equation~\eqref{eq:app-projection-bet-properties} follows.

Item (iv) combines the rounding result and directed squared-triangle
inequality in \cite[Lemmas~17 and~5]{appel2025regret}.  The
upward/downward rounding directions
are exactly those required by that result.
\end{proof}

We will also use two standard wealth facts.  Applying the exponential
supermartingale inequality of \cite[Lemma~A.4]{foster2021statistical} to
\(X_k=-\log b_k\) and rearranging shows that, for positive adapted
multipliers \(b_k\), with
\(\lambda_k=\mathbb E[b_k^{-1}\mid\mathcal F_{k-1}]\), one has, with
probability at least \(1-\delta\),
\begin{equation}
\sum_{k=1}^K-\log\lambda_k
\le
\log(1/\delta)+\sum_{k=1}^K\log b_k.
\label{eq:app-reciprocal-martingale}
\end{equation}
Second, the normalized multiplicative-weights recursion in
\cite[Lemma~6]{appel2025regret}, together with
\(\zeta_{K+1}(B)\le1\), gives, for a bettor \(B\) with positive prior
\(\zeta_1(B)\),
\begin{equation}
\sum_{k=1}^K\log b_{k,B}
\le
\log\frac1{\zeta_1(B)}+\sum_{k=1}^K\log Z_k,
\label{eq:app-wealth-identity}
\end{equation}
where \(Z_k\) is the episode-\(k\) market normalizer.  We use these two cited
facts without reproducing their proofs.

\subsection{Complete construction of the stage-tied market}
\label{app:learning-market}

The construction below adapts the custom tabular wealth-based estimator in
Appendices~L--N of \cite{appel2025regret}.  We retain their
directed-Hellinger fragment bets, calibration transport, full-support
reference bets, return-overprediction bet, backward row minimization, and
multiplicative wealth updates.  We change the statistical unit from a
statewise certificate to one shared layer-level certificate and use the
private original-scale code from Lemma~\ref{lem:app-deferred-coding}.
Because these changes alter both the bettor residuals and their entropy, the
construction and its bounds are proved below rather than quoted from their
statewise theorem.

Equip the compact policy space \(\Pi\) with its finite-dimensional product
topology.  Fix an episode \(k\), its pre-episode bettor posterior \(\zeta_k\),
and a queried policy \(\pi\).  We now define the exact backward predictor
\(M_{k,\pi}\).

Set
\(
\varepsilon_v=\varepsilon_c=\min\{1,8S/K\}
\).
Let \(G_v,G_c\subset[0,1]\) be uniform grids, including both endpoints,
whose mesh widths are at most \(\varepsilon_v\) and \(\varepsilon_c\),
respectively.  A stage-\(h\) certificate is
\begin{equation}
g=(h,a,v,c)\in\mathcal G_h
\coloneqq
\{h\}\times\mathcal A^{\mathcal S}
\times G_v^{\mathcal S}\times G_c^{\mathcal S}.
\label{eq:learning-certificate}
\end{equation}
The recommendation map \(a\) and threshold map \(c\) range over predecessor
states, while the continuation vector \(v\) is shared by the entire layer.
Since a scalar grid of mesh width at most \(\varepsilon\) can be chosen with
at most \(2+\varepsilon^{-1}\) points,
\begin{equation}
\log|\mathcal G_h|
\le
S\log A
+S\log(2+\varepsilon_v^{-1})
+S\log(2+\varepsilon_c^{-1}).
\label{eq:learning-certificate-entropy}
\end{equation}

For a certificate \(g=(h,a,v,c)\in\mathcal G_h\), write
\begin{align*}
a_g(s)&\coloneqq a(s),\\
C_g(s)&\coloneqq\Psi(v,c(s)),\\
x_g^{M,\pi}(s)
&\coloneqq
\pi_h(a_g(s)\mid s)
d_{C_g(s)}(M_h(s,a_g(s))).
\end{align*}
Define calibration potentials backward from the target layer by
\begin{align}
U_{g,h}^{M,\pi}(s)
&\coloneqq x_g^{M,\pi}(s),
\nonumber\\
\Phi_{g,h}^{M,\pi}(s,a')
&\coloneqq x_g^{M,\pi}(s),
\nonumber\\
\Phi_{g,j}^{M,\pi}(s,a')
&\coloneqq
\mathbb E_{(y,s')\sim M_j(s,a')}
[U_{g,j+1}^{M,\pi}(s')],
\qquad j<h,
\nonumber\\
U_{g,j}^{M,\pi}(s)
&\coloneqq
\sum_{a'\in\mathcal A}\pi_j(a'\mid s)
\Phi_{g,j}^{M,\pi}(s,a'),
\qquad j<h.
\label{eq:app-calibration-potentials}
\end{align}
Then
\(U_{g,1}^{M,\pi}(s_1)=\mathbb E_{M,\pi}[x_g^{M,\pi}(S_h)]\).

For the optimism bettor, define
\begin{align}
U_{\bullet,H+1}^{M,\pi}(s)&\coloneqq0,
\nonumber\\
\Phi_{\bullet,j}^{M,\pi}(s,a)
&\coloneqq
\mathbb E_{(y,s')\sim M_j(s,a)}
[y+U_{\bullet,j+1}^{M,\pi}(s')],
\nonumber\\
U_{\bullet,j}^{M,\pi}(s)
&\coloneqq
\sum_{a\in\mathcal A}\pi_j(a\mid s)
\Phi_{\bullet,j}^{M,\pi}(s,a).
\label{eq:app-pessimism-potentials}
\end{align}
Thus \(U_{\bullet,1}^{M,\pi}(s_1)\) is the predicted auxiliary return.

For every bettor \(B\), row \((j,s,a)\), and candidate law
\(\mu\in\Delta(\mathcal O)\), define its convex row objective as follows.
A fragment bettor for \(g\in\mathcal G_h\) uses
\begin{equation}
g_{B_g^{\rm frag}}^{j,s,a}(\mu)
\coloneqq
2d_{C_g(s)}(\mu)\mathbf 1\{j=h,\ a=a_g(s)\}.
\label{eq:app-fragment-objective}
\end{equation}
Its calibration companion uses
\begin{equation}
g_{B_g^{\rm cal}}^{j,s,a}(\mu)
\coloneqq
\begin{cases}
\displaystyle
\frac14\mathbb E_{(y,s')\sim\mu}
[U_{g,j+1}^{M,\pi}(s')],&j<h,\\[0.4em]
0,&j\ge h.
\end{cases}
\label{eq:app-calibration-objective}
\end{equation}
Let
\[
\eta\coloneqq
\min\!\left\{\frac12,
\sqrt{\frac{\log(32/\delta)}K}\right\},
\qquad
\lambda\coloneqq\eta/H.
\]
The optimism bettor uses
\begin{equation}
g_\bullet^{j,s,a}(\mu)
\coloneqq
\lambda\mathbb E_{(y,s')\sim\mu}
[y+U_{\bullet,j+1}^{M,\pi}(s')].
\label{eq:app-pessimism-objective}
\end{equation}
Finally, let \(u\) be the uniform law on \(\mathcal O\).  The reference
bettor attached to \((j_0,s_0,a_0)\) uses
\(2D_{\rm H}^2(\mu,u)\) at that row and zero at every other row.

Construct the predictor for \(j=H,H-1,\ldots,1\).  Once all later rows are
fixed, put
\begin{equation}
J_{k,\pi}^{j,s,a}(\mu)
\coloneqq
\sum_B\zeta_k(B)g_B^{j,s,a}(\mu)
\label{eq:app-aggregate-objective}
\end{equation}
and choose
\begin{equation}
M_{k,\pi,j}(s,a)
\in
\operatorname*{arg\,min}_{\mu\in\Delta(\mathcal O)}
J_{k,\pi}^{j,s,a}(\mu).
\label{eq:app-row-minimization}
\end{equation}
The recursion is acyclic: equations~\eqref{eq:app-calibration-objective} and
\eqref{eq:app-pessimism-objective} use only potentials constructed from later
stages.

Let \(N_{\mathcal G}=\sum_h|\mathcal G_h|\) and
\(\rho=\min\{1/8,1/(8K^4)\}\).  Initially assign total mass \(1/4\)
uniformly to fragment bettors, \(1/4\) uniformly to calibration bettors,
\(\rho\) uniformly to the \(HSA\) reference bettors, and
\(1/2-\rho\) to the optimism bettor.  Thus every selected fragment or
calibration witness has prior at least
\begin{equation}
w_{\min}\coloneqq\frac1{4N_{\mathcal G}}.
\label{eq:app-min-prior}
\end{equation}

After executing \(\pi_k\), consider a certificate
\(g=(h,a,v,c)\in\mathcal G_h\) and put \(M=M_{k,\pi_k}\).  Its fragment
witness is inactive, with multiplier one, unless
\(A_{k,h}=a(S_{k,h})\).  On that event, let
\begin{align*}
p&=M_h(S_{k,h},A_{k,h}),\\
q&=q_{\Psi(v,c(S_{k,h}))}(p),\\
d&=d_{\Psi(v,c(S_{k,h}))}(p).
\end{align*}
With \(O_{k,h}=(Y_{k,h},S_{k,h+1})\), its active multiplier is
\begin{equation}
b_{k,g}^{\rm frag}(O_{k,h})
=\sqrt{\frac{q(O_{k,h})}{p(O_{k,h})}}+d.
\label{eq:learning-fragment-bet}
\end{equation}
The associated layer-calibration multiplier is
\begin{equation}
\begin{split}
b_{k,g}^{\rm cal}
&=1+\frac14
\left\{
\mathbb E_{M,\pi_k}[x_g^{M,\pi_k}(S_h)]
-x_g^{M,\pi_k}(S_{k,h})
\right\}\\
&\in[3/4,5/4].
\end{split}
\label{eq:learning-calibration-bet}
\end{equation}
The reference multiplier is the same projection multiplier with \(C=\{u\}\) when
its row is visited and is one otherwise.  For any query \(\pi\), write
\[
\widehat G_k(\pi)
\coloneqq U_{\bullet,1}^{M_{k,\pi},\pi}(s_1).
\]
After sampling \(\pi_k\), abbreviate
\(G_k^Y=\sum_hY_{k,h}\) and
\(\widehat G_k=\widehat G_k(\pi_k)\).  The optimism
multiplier is
\begin{equation}
b_{k,\bullet}
=1+\frac\eta H(\widehat G_k-G_k^Y)
\in[1/2,3/2].
\label{eq:app-pessimism-multiplier}
\end{equation}
Finally update
\begin{equation}
Z_k\coloneqq\sum_B\zeta_k(B)b_{k,B},
\qquad
\zeta_{k+1}(B)
\coloneqq\frac{\zeta_k(B)b_{k,B}}{Z_k}.
\label{eq:app-posterior-update}
\end{equation}

\begin{lemma}[Existence, uniqueness, and full support]
\label{lem:app-market-existence}
At every episode, equation~\eqref{eq:app-row-minimization} has a unique
full-support minimizer at every row and for every queried policy.  All bettor
multipliers and \(Z_k\) are finite and strictly positive, and every reference
bettor retains positive posterior mass.
\end{lemma}

\begin{proof}
We argue by induction over episodes.  Every reference bettor has positive prior
mass in episode one.  Suppose its posterior mass is positive at the start of
episode \(k\).  The objective in equation~\eqref{eq:app-aggregate-objective}
is continuous on the compact simplex, so a minimizer exists.  The reference
term at the current row has positive coefficient and is a positive multiple of
\[
2D_{\rm H}^2(\mu,u)
=2-2\sum_{o\in\mathcal O}\sqrt{\mu(o)u(o)}.
\]
Since \(u(o)>0\) and \(z\mapsto-\sqrt z\) is strictly convex, this term is
strictly convex.  Every other term is convex or affine.  Their sum is strictly
convex, so the minimizer is unique.

Let \(p\) be that minimizer.  Suppose \(p(o_0)=0\).  Choose \(o_1\) with
\(p(o_1)>0\), and for sufficiently small \(t>0\) put
\(p_t=p+t(e_{o_0}-e_{o_1})\).  The unscaled reference objective changes by
\begin{align*}
&2D_{\rm H}^2(p_t,u)-2D_{\rm H}^2(p,u)\\
&\quad=
-2\sqrt{u(o_0)}\sqrt t
\\
&\qquad
+2\sqrt{u(o_1)}
\frac{t}{\sqrt{p(o_1)}+\sqrt{p(o_1)-t}}
\\
&\quad
\le-c\sqrt t+C_1t
\end{align*}
for constants \(c>0\) and \(C_1<\infty\).  Here is the promised uniform
boundary bound for every other directed Hellinger term.  Fix a compact convex
row \(C\), and let \(q_C(p)\) be a projection of \(p\).  Using that same
\(q_C(p)\) as a feasible comparison for \(p_t\) gives, for
\(0<t\le p(o_1)/2\),
\begin{align*}
d_C(p_t)-d_C(p)
&\le D_{\rm H}^2(p_t,q_C(p))-D_{\rm H}^2(p,q_C(p))\\
&=-\sqrt{q_C(p)(o_0)}\sqrt t\\
&\quad+\sqrt{q_C(p)(o_1)}
\{\sqrt{p(o_1)}-\sqrt{p(o_1)-t}\}\\
&\le \frac{t}{\sqrt{p(o_1)}}.
\end{align*}
Thus each such term increases by at most a constant times \(t\), including
when its projection gives positive mass to the newly opened coordinate.
Every affine term also changes by at most a constant times \(t\).  Because
the bettor family is finite and the reference coefficient is strictly
positive, the total objective change is at most
\(-c'\sqrt t+C't<0\) for all sufficiently small \(t\).  This
contradicts optimality.  Therefore \(p\) has full support.

Full support makes every projection ratio finite.  Fragment and reference
multipliers are strictly positive by
Lemma~\ref{lem:app-imported-hellinger}; calibration multipliers belong to
\([3/4,5/4]\), and the optimism multiplier belongs to \([1/2,3/2]\).
Hence \(Z_k>0\), and equation~\eqref{eq:app-posterior-update} preserves the
positive mass of every reference bettor.  This closes the induction.
\end{proof}

\begin{proposition}[Continuity and policy coherence]
\label{prop:app-predictor-continuity}
For every episode \(k\), the map
\(\pi\mapsto M_{k,\pi}\) is continuous.  Its induced finite trajectory law
is continuous and Borel measurable in \(\pi\).  It is policy coherent in the
trajectory-law sense required by \cite{appel2025regret}: for every query
\(\pi\), there is a nonanticipating outcome selector
\(\sigma^{k,\pi}\) such that the predicted law is exactly
\(\sigma^{k,\pi}\bowtie\pi\), including at histories that have zero predicted
probability.
\end{proposition}

\begin{proof}
Proceed backward in \(j\).  At \(j=H\), every row objective is a finite
weighted sum of continuous convex Hellinger terms and affine terms.  Its
coefficients depend continuously on the finitely many action probabilities in
\(\pi\).  Suppose continuity has been proved for all later stages.  The finite
potential recursions in equations~\eqref{eq:app-calibration-potentials} and
\eqref{eq:app-pessimism-potentials} use only sums and products of later rows and
policy probabilities, so they preserve continuity.  Hence the stage-\(j\)
objective is jointly continuous in \((\mu,\pi)\).

Let \(\pi_n\to\pi\), and let \(\mu_n\) be the unique minimizer of one row
objective at \(\pi_n\).  Compactness provides a convergent subsequence
\(\mu_{n_\ell}\to\bar\mu\).  For every feasible \(\mu\), joint continuity
and optimality give
\[
J_{k,\pi}^{j,s,a}(\bar\mu)
=\lim_\ell J_{k,\pi_{n_\ell}}^{j,s,a}(\mu_{n_\ell})
\le
\lim_\ell J_{k,\pi_{n_\ell}}^{j,s,a}(\mu)
=J_{k,\pi}^{j,s,a}(\mu).
\]
Thus \(\bar\mu\) is the limiting minimizer.  Uniqueness from
Lemma~\ref{lem:app-market-existence} identifies it with
\(M_{k,\pi,j}(s,a)\).  Every convergent subsequence has this limit, so the full
sequence converges.  This completes the backward induction.

The probability of a finite action/outcome trajectory is a finite product of
policy probabilities and predicted-row probabilities, so the induced
trajectory law is continuous and therefore Borel.  To verify the stronger
policy-coherence statement rather than only its action-kernel consequence,
define, for every post-action history \(\eta_j\) ending in state \(s_j\),
\begin{equation}
\sigma^{k,\pi}_j(\cdot\mid\eta_j,a)
\coloneqq M_{k,\pi,j}(s_j,a).
\label{eq:app-predictor-selector}
\end{equation}
This is a Borel, nonanticipating Markov selector, specified at null histories
as well as reached ones.  Ionescu--Tulcea recursion (a finite product here)
shows that alternating the action kernels \(\pi_j(\cdot\mid s_j)\) with
\eqref{eq:app-predictor-selector} produces exactly the trajectory law used in
the definition of \(M_{k,\pi}\).  Hence that law is
\(\sigma^{k,\pi}\bowtie\pi\), which is the policy-coherence hypothesis of the
imported DEC theorem.
\end{proof}

\subsection{Market supernormalization}
\label{app:learning-supernormalization}

Appendix~M of \cite{appel2025regret} proves supernormalization for their
statewise bettor family.  Our layer-aggregated calibration multiplier has a
different residual and telescoping identity, so the analogous normalizer
bound is established here for the modified market.

\begin{lemma}[Exact telescoping identities]
\label{lem:app-telescoping}
For this lemma write \(G^Y\coloneqq\sum_{j=1}^H Y_j\).
For a calibration bettor with target stage \(h\),
\begin{align}
&\mathbb E_{M,\pi}[x_g^{M,\pi}(S_h)]-x_g^{M,\pi}(S_h)
\nonumber\\
&=U_{g,1}(s_1)-\Phi_{g,1}(S_1,A_1)
\nonumber\\
&\quad+
\sum_{j=1}^{h-1}
\{\Phi_{g,j}(S_j,A_j)-U_{g,j+1}(S_{j+1})\}
\nonumber\\
&\quad+
\sum_{j=1}^{h-1}
\{U_{g,j+1}(S_{j+1})-\Phi_{g,j+1}(S_{j+1},A_{j+1})\}.
\label{eq:app-calibration-telescope}
\end{align}
For the optimism bettor,
\begin{align}
\widehat G-G^Y
&=U_{\bullet,1}(s_1)-\Phi_{\bullet,1}(S_1,A_1)
\nonumber\\
&\quad+
\sum_{j=1}^H
\{\Phi_{\bullet,j}(S_j,A_j)-Y_j
-U_{\bullet,j+1}(S_{j+1})\}
\nonumber\\
&\quad+
\sum_{j=1}^{H-1}
\{U_{\bullet,j+1}(S_{j+1})
-\Phi_{\bullet,j+1}(S_{j+1},A_{j+1})\}.
\label{eq:app-pessimism-telescope}
\end{align}
\end{lemma}

\begin{proof}
Expand the two sums in equation~\eqref{eq:app-calibration-telescope}.  Every
intermediate \(\Phi_{g,j}\) and \(U_{g,j}\) appears once with each sign.
The surviving terms are
\(U_{g,1}(s_1)=\mathbb E_{M,\pi}x_g(S_h)\) and
\(-\Phi_{g,h}(S_h,A_h)=-x_g(S_h)\).  The same cancellation in
equation~\eqref{eq:app-pessimism-telescope} leaves
\(U_{\bullet,1}(s_1)=\widehat G\),
\(-\sum_jY_j=-G^Y\), and
\(-U_{\bullet,H+1}=0\).
\end{proof}

\begin{lemma}[Supernormalization]
\label{lem:app-supernormalization}
For every episode \(k\), every predictable policy law, and every feasible
nonanticipating actual selector,
\begin{equation}
\mathbb E[Z_k\mid\mathcal F_{k-1}]\le1.
\label{eq:app-one-step-supernormalization}
\end{equation}
Consequently, for every \(\delta\in(0,1)\), with probability at least
\(1-\delta\),
\begin{equation}
\sum_{k=1}^K\log Z_k\le\log(1/\delta).
\label{eq:app-normalizer-bound}
\end{equation}
\end{lemma}

\begin{proof}
Fix the past and first fix a queried policy \(\pi\).  Let
\(M=M_{k,\pi}\).  At a reached post-action history at stage \(j\), write
\(p_j=M_j(S_j,A_j)\) and let \(\nu_j\) be the actual conditional law of
\((Y_j,S_{j+1})\).  The law \(\nu_j\) may depend on the full
history.

For each bettor, decompose \(b_{k,B}-1\) into outcome increments and action
increments.  For fragment and reference bettors, the outcome increment is the
active projection multiplier minus one.  For a calibration bettor targeted at
stage \(h\), use
\begin{align*}
\xi_{B,j}
&=\frac14
\{\Phi_{g,j}(S_j,A_j)-U_{g,j+1}(S_{j+1})\}
\mathbf1\{j<h\},\\
\rho_{B,j}
&=\frac14
\{U_{g,j}(S_j)-\Phi_{g,j}(S_j,A_j)\}
\mathbf1\{j\le h\}.
\end{align*}
For the optimism bettor, use
\begin{align*}
\xi_{B,j}
&=\lambda\{\Phi_{\bullet,j}(S_j,A_j)-Y_j
-U_{\bullet,j+1}(S_{j+1})\},\\
\rho_{B,j}
&=\lambda\{U_{\bullet,j}(S_j)-\Phi_{\bullet,j}(S_j,A_j)\}.
\end{align*}
Lemma~\ref{lem:app-telescoping} verifies that the increments sum exactly to
the corresponding episode multiplier minus one.

Given the history immediately before the action, the conditional mean of every
\(\rho_{B,j}\) is zero because \(A_j\sim\pi_j(\cdot\mid S_j)\).  Given the
post-action history, equations~\eqref{eq:app-bet-derivative},
\eqref{eq:app-calibration-objective}, and
\eqref{eq:app-pessimism-objective} give, for every bettor type,
\begin{equation}
\mathbb E[\xi_{B,j}\mid\text{post-action history}]
=-Dg_B^{j,S_j,A_j}(p_j)[\nu_j-p_j].
\label{eq:app-increment-derivative}
\end{equation}

The row \(p_j\) minimizes the convex function
\(J_{k,\pi}^{j,S_j,A_j}\).  For any feasible direction \(\nu_j-p_j\), put
\(\phi(t)=J((1-t)p_j+t\nu_j)\).  Minimality gives
\((\phi(t)-\phi(0))/t\ge0\) for \(t>0\); letting \(t\downarrow0\) gives
\begin{equation}
DJ_{k,\pi}^{j,S_j,A_j}(p_j)[\nu_j-p_j]\ge0.
\label{eq:app-row-first-order}
\end{equation}
Average equation~\eqref{eq:app-increment-derivative} over the finite bettor
population, use equation~\eqref{eq:app-row-first-order}, and then use the tower
property.  This yields
\begin{align*}
&\mathbb E[Z_k-1\mid\mathcal F_{k-1},\pi]\\
&\quad=-\mathbb E\!\left[
\sum_{j=1}^H
DJ_{k,\pi}^{j,S_j,A_j}(p_j)[\nu_j-p_j]
\,\middle|\,\mathcal F_{k-1},\pi
\right]
\le0.
\end{align*}
Proposition~\ref{prop:app-predictor-continuity} supplies the Borel measurability
needed to integrate this inequality against the predictable law \(p_k\), proving
equation~\eqref{eq:app-one-step-supernormalization}.

The product \(\mathcal Z_t=\prod_{k=1}^tZ_k\) is therefore a nonnegative
supermartingale with \(\mathcal Z_0=1\).  Markov's inequality gives
\(\mathbb P(\mathcal Z_K>1/\delta)\le\delta\).  Every \(Z_k\) is positive by
Lemma~\ref{lem:app-market-existence}, so taking logarithms on the complementary
event proves equation~\eqref{eq:app-normalizer-bound}.
\end{proof}

\subsection{Fragment control, layer calibration, and inaccuracy}
\label{app:learning-inaccuracy}

Fix the candidate \(N^\circ\) from the statement and choose one of its
unit-coordinate certificate tuples
\[
(a_h^\circ,v_{h+1}^\circ,c_h^\circ)_{h=1}^H
\]
in the parameterization of Section~\ref{sec:learning-target}.
Round \(v_{h+1}^\circ\) coordinatewise upward and
\(c_h^\circ\) coordinatewise downward to form
\(g_h^\circ\in\mathcal G_h\).  By
equation~\eqref{eq:app-rounded-inclusion}, every recommended row of
\(N^\circ\) is
contained in the rounded halfspace, because
\(\Psi(v_{h+1}^\circ,c_h^\circ(s))
=\Psi_H(Hv_{h+1}^\circ,Hc_h^\circ(s))\).

For a queried policy \(\pi\), put
\[
x_{k,h}^\pi(s)
\coloneqq x_{g_h^\circ}^{M_{k,\pi},\pi}(s).
\]
Define cumulative predicted and factual scores
\begin{align}
P_h
&\coloneqq
\sum_{k=1}^K\mathbb E_{\pi\sim p_k}
\mathbb E_{M_{k,\pi},\pi}[x_{k,h}^\pi(S_h)\mid\mathcal F_{k-1}],
\label{eq:app-predicted-score}\\
F_h
&\coloneqq
\sum_{k=1}^K\mathbb E_{\pi\sim p_k}
\mathbb E_{\rm actual,\pi}[x_{k,h}^\pi(S_h)\mid\mathcal F_{k-1}].
\label{eq:app-factual-score}
\end{align}

\begin{lemma}[Fragment control]
\label{lem:app-fragment-control}
For every fixed \(h\) and \(\delta\in(0,1)\), with probability at least
\(1-\delta\),
\begin{equation}
F_h
\le
\log(1/\delta)+\log(1/w_{\min})
+\sum_{k=1}^K\log Z_k.
\label{eq:app-fragment-control}
\end{equation}
\end{lemma}

\begin{proof}
Use the fragment bettor for \(g_h^\circ\).  Condition on the past, queried
policy, and public history immediately before \(A_h\) is drawn.  If the current
state is \(s\), the recommended action is selected with probability
\(\pi_h(a_h^\circ(s)\mid s)\).  If another action is selected, the multiplier
is one.  If the recommended action is selected, the actual conditional row lies
in the rounded halfspace.  Equation~\eqref{eq:app-projection-bet-properties}
therefore gives
\[
\mathbb E[(b_{k}^{\rm frag})^{-1}\mid\text{pre-action history}]
\le1-x_{k,h}^\pi(s).
\]
Average over the factual history and \(\pi\sim p_k\).  If
\(\bar x_{k,h}\) is the resulting conditional expected score, then the
reciprocal conditional mean \(\lambda_k\) satisfies
\(\lambda_k\le1-\bar x_{k,h}\).  Since \(-\log\) is decreasing and
\(-\log(1-z)\ge z\) for \(z\in[0,1)\),
\(-\log\lambda_k\ge\bar x_{k,h}\).  Summing and applying
equations~\eqref{eq:app-reciprocal-martingale} and
\eqref{eq:app-wealth-identity}, together with
\(\zeta_1(B)\ge w_{\min}\), gives equation~\eqref{eq:app-fragment-control}.
\end{proof}

\begin{lemma}[One-layer calibration transport]
\label{lem:app-calibration-control}
For every fixed \(h\) and \(\delta\in(0,1)\), with probability at least
\(1-\delta\),
\begin{equation}
P_h
\le
2F_h+6\left\{
\log(1/\delta)+\log(1/w_{\min})
+\sum_{k=1}^K\log Z_k
\right\}.
\label{eq:app-calibration-control}
\end{equation}
\end{lemma}

\begin{proof}
For a queried policy \(\pi\), let
\[
u_k(\pi)=\mathbb E_{M_{k,\pi},\pi}[x_{k,h}^\pi(S_h)],
\qquad
X_k(\pi)=x_{k,h}^\pi(S_h).
\]
The calibration multiplier is
\(b_k=1+\{u_k(\pi_k)-X_k(\pi_k)\}/4\).
For every \(y\in[-1,1]\), direct multiplication by the positive number
\(1+y/4\) verifies
\begin{equation}
\frac1{1+y/4}
\le1-\frac y4+\frac{|y|}{12}.
\label{eq:app-scalar-reciprocal}
\end{equation}
For completeness, when \(y\in[0,1]\), subtracting the left side after this
multiplication leaves \(y(2-y)/24\ge0\).  When \(y\in[-1,0]\), it leaves
\((-y-y^2)/12\ge0\).  These two cases cover the entire stated interval.
Because \(u_k(\pi),X_k(\pi)\in[0,1]\),
\(|u_k-X_k|\le u_k+X_k\).  Define the predictable averages
\begin{align*}
\bar u_k
&=\mathbb E_{\pi\sim p_k}u_k(\pi),\\
\bar x_k
&=\mathbb E_{\pi\sim p_k}
\mathbb E_{\rm actual,\pi}
[X_k(\pi)\mid\mathcal F_{k-1}].
\end{align*}
Taking conditional expectations in equation~\eqref{eq:app-scalar-reciprocal}
gives
\[
\lambda_k
\le1-\frac16\bar u_k+\frac13\bar x_k.
\]
Put \(z_k=\bar u_k/6-\bar x_k/3\).  Since
\(z_k\in[-1/3,1/6]\), monotonicity of \(-\log\) and
\(-\log(1-z)\ge z\) give
\[
-\log\lambda_k\ge\frac16\bar u_k-\frac13\bar x_k.
\]
Sum over \(k\), use equations~\eqref{eq:app-reciprocal-martingale} and
\eqref{eq:app-wealth-identity}, and then multiply by six.  The identities
\(\sum_k\bar u_k=P_h\) and \(\sum_k\bar x_k=F_h\) yield
equation~\eqref{eq:app-calibration-control}.
\end{proof}

\begin{proof}[Proof of Theorem~\ref{thm:stage-tied-estimation}]
For a nonrecommended action, the true surrogate row is
\(\Delta(\mathcal O)\), so its directed loss is zero.  For a recommended
action, equation~\eqref{eq:app-rounded-loss} gives
\begin{align*}
&d_{N_h^\circ(s,a_h^\circ(s))}
(M_{k,\pi,h}(s,a_h^\circ(s)))\\
&\qquad\le
2d_{C_{g_h^\circ}(s)}
(M_{k,\pi,h}(s,a_h^\circ(s)))
+2(\varepsilon_v+\varepsilon_c).
\end{align*}
Average over the predicted state and action occupancy.  The probability of the
recommended action is exactly the factor in the layer score.  Summing over all
episodes and stages yields
\begin{equation}
\operatorname{Est}_K
\le
2\sum_{h=1}^HP_h+2KH(\varepsilon_v+\varepsilon_c).
\label{eq:app-est-vs-predicted-score}
\end{equation}

Allocate failure probability \(\delta/(4H)\) to each of the \(2H\)
applications of Lemmas~\ref{lem:app-fragment-control} and
\ref{lem:app-calibration-control}, and allocate \(\delta/2\) to
Lemma~\ref{lem:app-supernormalization}.  The union bound gives a simultaneous
event of probability at least \(1-\delta\).  On this event, for every \(h\),
\begin{align*}
P_h
&\le2F_h+6\left\{
\log\frac{4H}\delta+\log(1/w_{\min})+\sum_k\log Z_k
\right\}\\
&\le8\left\{
\log\frac{4H}\delta+\log(1/w_{\min})+\log\frac2\delta
\right\}.
\end{align*}
The second line substitutes the fragment bound and
\(\sum_k\log Z_k\le\log(2/\delta)\).

By equations~\eqref{eq:learning-certificate-entropy} and
\eqref{eq:app-min-prior},
\begin{align*}
\log(1/w_{\min})
&=\log(4N_{\mathcal G})\\
&\le\log(4H)+S\log A\\
&\quad+S\log(2+\varepsilon_v^{-1})\\
&\quad+S\log(2+\varepsilon_c^{-1}).
\end{align*}
Substituting this inequality into
equation~\eqref{eq:app-est-vs-predicted-score} gives the explicit bound
\begin{align}
\operatorname{Est}_K
\le{}&16H\left[
S\log A+S\log(2+\varepsilon_v^{-1})
\right.
\nonumber\\[-0.2em]
&\left.\qquad
+S\log(2+\varepsilon_c^{-1})
+\log\frac{64H^2}{\delta^2}
\right]
\nonumber\\
&+2KH(\varepsilon_v+\varepsilon_c).
\label{eq:app-explicit-inaccuracy}
\end{align}

If \(K\ge8S\), set
\(\varepsilon_v=\varepsilon_c=8S/K\).  The rounding term is then
\(32HS\).  Moreover,
\(2+K/(8S)\le2(K+1)\).  Also,
\(\log A\le\Lambda_K(\delta)\),
\(\log(64H^2/\delta^2)\le\Lambda_K(\delta)\), and
\[
\log(2+K/(8S))
\le\log2+\log(K+1)
\le2\Lambda_K(\delta).
\]
Since \(S\ge1\), the bracket in
equation~\eqref{eq:app-explicit-inaccuracy} is therefore at most
\(6S\Lambda_K(\delta)\).  Its contribution is at most
\(96HS\Lambda_K(\delta)\).  The rounding contribution is at most
\(32HS\Lambda_K(\delta)\), because \(\Lambda_K(\delta)>1\).  If
\(K<8S\), every squared Hellinger loss is at most one, so
\(\operatorname{Est}_K\le HK<8HS\le144HS\Lambda_K(\delta)\).  In either
case,
\[
\operatorname{Est}_K\le144HS\Lambda_K(\delta),
\]
which completes the proof.
\end{proof}

\subsection{Original-scale optimism}
\label{app:learning-optimism}

This subsection adapts the return-overprediction bettor used in the proofs of
Theorems~2 and~5 of \cite{appel2025regret}.  The coded sum is controlled
in unit coordinates, and the final inequality is then returned to the task
scale; this is why the proof does not require the predicted task-scale value
to be bounded by \(H\).

\begin{lemma}[Original-scale optimism control]
\label{lem:app-optimism}
Let \(m_k(\pi)\) be the task-scale conditional factual return defined in
equation~\eqref{eq:actual-return}.  With probability at least
\(1-\delta\), the optimism witness guarantees
\begin{equation}
\begin{split}
\sum_{k=1}^{K}\mathbb E_{\pi\sim p_k}
&\bigl[\widehat v_k(\pi)-m_k(\pi)\bigr]\\
&\le 8H^2\Bigl\{\sqrt{K\log(32/\delta)}
+\log(32/\delta)\Bigr\}.
\end{split}
\label{eq:learning-optimism-bound}
\end{equation}
\end{lemma}

\begin{proof}
Let \(L_\delta=\log(32/\delta)\) and
\(\eta=\min\{1/2,\sqrt{L_\delta/K}\}\).  The optimism bettor has prior
at least \(3/8\).  Apply the deterministic wealth identity
\eqref{eq:app-wealth-identity}, and apply the normalizer bound
\eqref{eq:app-normalizer-bound} with failure probability \(\delta/2\).
On the resulting event,
\begin{equation}
\sum_{k=1}^K
\log\!\left(1+\frac\eta H \Xi_k\right)
\le L_\delta,
\qquad
\Xi_k\coloneqq\widehat G_k-G_k^Y\in[-H,H].
\label{eq:app-pessimism-log-wealth}
\end{equation}
For \(|x|\le1/2\),
\(\log(1+x)\ge x-x^2\).  Apply this with \(x=\eta\Xi_k/H\), use
\(\Xi_k^2\le H^2\), and multiply by \(H/\eta\).  Then
\begin{equation}
\sum_{k=1}^K\Xi_k
\le\frac{HL_\delta}{\eta}+\eta HK.
\label{eq:app-realized-optimism}
\end{equation}

For a fixed policy \(\pi\), equation
\eqref{eq:app-conversion-expectation} identifies
\(m_k(\pi)/H\) with the conditional expected coded sum.  Define the
unit-coordinate predictable gap
\[
\bar\Xi_k
\coloneqq
\mathbb E_{\pi\sim p_k}
[\widehat G_k(\pi)-m_k(\pi)/H].
\]
Predictability of \(p_k\) and policy coherence give
\(\bar\Xi_k=\mathbb E[\Xi_k\mid\mathcal F_{k-1}]\).  Hence
\(\bar\Xi_k-\Xi_k\) is a martingale difference bounded in absolute value by
\(2H\).  Azuma-Hoeffding therefore gives, with probability at least
\(1-\delta/2\),
\begin{equation}
\sum_{k=1}^K(\bar\Xi_k-\Xi_k)
\le2H\sqrt{2K\log(2/\delta)}.
\label{eq:app-optimism-azuma}
\end{equation}
Combining equations~\eqref{eq:app-realized-optimism} and
\eqref{eq:app-optimism-azuma} yields
\[
\sum_k\bar\Xi_k
\le
H\left\{
\frac{L_\delta}\eta+\eta K+2\sqrt{2K\log(2/\delta)}
\right\}.
\]
If \(\eta=\sqrt{L_\delta/K}\), the bracket is at most
\(5\sqrt{KL_\delta}\).  If clipping gives \(\eta=1/2\), direct substitution
is bounded by \(8\{\sqrt{KL_\delta}+L_\delta\}\).  The same latter bound
also covers the first case.  Finally,
\(\widehat v_k(\pi)=H\widehat G_k(\pi)\), so multiplying by \(H\)
proves equation~\eqref{eq:learning-optimism-bound}.  Notice that
\(\widehat v_k(\pi)\) itself may be as large as \(H^2\); the proof uses
bounded coded sums, not an upper bound of \(H\) on the predictor.
\end{proof}

\subsection{Applicability of the robust DEC theorem}
\label{app:learning-dec}

\begin{lemma}[Regularity of the stage-tied decision problem]
\label{lem:app-class-regularity}
Give the certificate space
\begin{equation}
\Theta
\coloneqq
\prod_{h=1}^H
\bigl(
\mathcal A^{\mathcal S}
\times[0,H]^{\mathcal S}
\times[0,H]^{\mathcal S}
\bigr)
\label{eq:app-certificate-parameter-space}
\end{equation}
the product of the discrete and Euclidean topologies, and regard candidates
as certificate tuples, as in Section~\ref{sec:learning-target}.  Then:
\begin{enumerate}[label=\textup{(\roman*)},leftmargin=2em]
\item \(\mathfrak H_{\rm st}^{(H)}\) is a compact, hence standard Borel,
subset of \(\Theta\).
\item For every continuous policy-coherent predictor \(M\),
\((N,\pi)\mapsto L(M,N,\pi)\) and
\(\pi\mapsto\widehat v_M^H(\pi)\) are continuous, and
\(N\mapsto v_N^\star\) is continuous on
\(\mathfrak H_{\rm st}^{(H)}\).
\item After any fixed pre-episode history, each localized class
\(\mathcal H_k\) is compact and Borel.  The objective \(Q_k\) is a bounded
Borel function on the compact metrizable space \(\Delta(\Pi)\); its inner
supremum is attained.
\end{enumerate}
Consequently, all integrals and optimizations in
equations~\eqref{eq:learning-localized-class} and
\eqref{eq:learning-fuzzy-objective} are well-defined.  The measurable
approximate policy selector is supplied by
Assumption~\ref{ass:decision-oracle}: that assumption includes
\(\mathcal F_{k-1}\)-measurability of the returned \(p_k\), rather than
claiming that the cited DEC theorems furnish an implementation or a measurable
selection algorithm for this class.
\end{lemma}

\begin{proof}
The ambient space \(\Theta\) is compact and metrizable.  First consider one
recommended row.  With
\(z_v(y,s')=Hy+v(s')\), its graph is
\[
\left\{(v,c,\nu):
\nu\in\Delta(\mathcal O),\ 
\langle z_v,\nu\rangle\ge c
\right\},
\]
which is closed.  The correspondence
\((v,c)\mapsto\Psi_H(v,c)\) is also lower hemicontinuous.  To see this, let
\((v_n,c_n)\to(v,c)\) and \(\nu\in\Psi_H(v,c)\).  If \(\nu\) is not already
feasible for the \(n\)-th halfspace, mix it with a point mass on an outcome
\((1,s_n)\) maximizing \(v_n(s_n)\).  The required mixing weight tends to
zero because the feasibility deficit tends to zero.  The only possible
zero-margin limit has \(c=H\) and \(v\equiv0\); in that case feasibility of
\(\nu\) forces it to be supported on \(Y=1\), making it feasible for every
nearby parameter since \(v_n\ge0\) and \(c_n\le H\).  Thus the mixing argument
also covers the boundary case.  The row correspondence is therefore
continuous and compact-valued.  Nonrecommended rows are the constant
correspondence \(\Delta(\mathcal O)\), and recommendation maps live in a
finite discrete space, so every stage-tied row has the same regularity.

It remains to show that the \(H\)-bounded restriction is closed.  Let
\(\mathcal K\) be the set of pairs \((N,\sigma)\), with
\(N\in\Theta\) and
\(\sigma\in\Delta(\mathcal O)^{HSA}\), such that every
\(\sigma_{h,s,a}\in N_h(s,a)\) and
\begin{equation}
\max_{(h,s,a)}\ \max_{\pi\in\Pi}
\mathbb E_{\sigma,\pi}
\left[\sum_{t=h}^HHY_t\,\middle|\,
S_h=s,A_h=a\right]
\le H.
\label{eq:app-h-bounded-closed-condition}
\end{equation}
Row-graph closedness makes the selector constraints closed.  The finite
horizon expectation in \eqref{eq:app-h-bounded-closed-condition} is jointly
continuous in \((\sigma,\pi)\), and maximizing it over the compact policy
space and the finitely many starting triples preserves continuity in
\(\sigma\).  Hence \(\mathcal K\) is closed in a compact space.  Its projection
onto \(\Theta\) is compact and is exactly
\(\mathfrak H_{\rm st}^{(H)}\), proving (i).

Directed squared Hellinger distance is continuous in both probability laws.
The row correspondence just proved continuous, so Berge's maximum theorem
implies that
\((p,v,c)\mapsto d_{\Psi_H(v,c)}(p)\) is continuous.  Predicted occupancy is a
finite product of the continuous row and policy probabilities.  Summing its
products with these row losses proves continuity of \(L\).  The same finite
recursion proves continuity of \(\widehat v_M^H\).  Finally, backward robust
dynamic programming expresses \(v_N^\star\) through finitely many operations
of the form
\[
\max_a\ \min_{\nu\in N_h(s,a)}
\mathbb E_\nu[HY+V_{h+1}(S')].
\]
Another induction using Berge's theorem, followed by the finite maximum over
actions, proves continuity of \(v_N^\star\).  This proves (ii).

For fixed past data, each prefix loss in
\eqref{eq:learning-localized-class} is continuous in \(N\), by (ii) and
bounded convergence under \(p_i\).  Thus \(\mathcal H_k\) is a closed subset
of the compact class in (i).  Identify a subprobability measure on
\(\mathcal H_k\) with a probability measure on the compact space obtained by
adding one cemetery point; this makes
\(\Delta_{\le1}(\mathcal H_k)\) compact and metrizable.  For
\((p,\mu)\) in the product of the two measure spaces, both
\[
(p,\mu)\longmapsto
\mathbb E_{N\sim\mu,\pi\sim p}L(M,N,\pi)
\quad\text{and}\quad
(p,\mu)\longmapsto
\mathbb E_{N\sim\mu,\pi\sim p}
[v_N^\star-\widehat v_M^H(\pi)]
\]
are continuous.  The first therefore defines a nonempty compact feasible set
of \(\mu\)'s for every \(p\), with the zero subprobability always feasible.
Since \(0\le v_N^\star\le H\) and \(\widehat v_M^H(\pi)\ge0\), the zero
subprobability also shows \(0\le Q_k(p)\le H\).  The compact-maximum theorem
gives attainment of the inner supremum and upper
semicontinuity, hence Borel measurability, of \(Q_k\).  The market recursion
uses finite Borel operations and unique minimizers of jointly continuous
compact problems, so these conclusions also hold episodewise as functions of
the realized pre-episode data.  Assumption~\ref{ass:decision-oracle} then
postulates a predictable measurable approximate selector, as stated.
\end{proof}

The next definition is the localized, task-scale form of the fuzzy robust DEC
of \cite[Definition~2]{appel2025regret}; only the class, loss, and payoff
scale have been adapted.

\begin{definition}[Fuzzy robust DEC]
\label{def:learning-fuzzy-dec}
For a candidate class \(\mathcal H\), an admissible predictor
\(M\in\mathfrak M_{\rm adm}\), and radius \(\varepsilon>0\), define the
task-scale predicted value
\[
\widehat v_M^H(\pi)
\coloneqq H\mathbb E_{M_\pi,\pi}\!\left[\sum_{h=1}^H Y_h\right].
\]
The localized original-scale fuzzy robust DEC is
\begin{equation}
\begin{split}
\operatorname{dec}^{\rm f,H}_\varepsilon(\mathcal H,M)
\coloneqq
\inf_{p\in\Delta(\Pi)}
\sup_{\substack{\mu\in\Delta_{\le1}(\mathcal H):\\
\mathbb E_{N\sim\mu,\pi\sim p}L(M,N,\pi)\le\varepsilon^2}}\\
\mathbb E_{N\sim\mu,\pi\sim p}
\bigl[v^\star_N-\widehat v_M^H(\pi)\bigr],
\end{split}
\label{eq:learning-dec-definition}
\end{equation}
where \(\Delta_{\le1}\) is the set of Borel subprobability measures and
integration with respect to \(\mu\) is unnormalized.  The adjective
``fuzzy'' refers to the mixture \(\mu\), not to uncertainty about
\(B^\star\); ``localized'' refers to restricting the class.  The
episode-\(k\) objective \(Q_k\) of the main text is the inner program of
\eqref{eq:learning-dec-definition} at
\((\mathcal H,M)=(\mathcal H_k,\widehat M_k)\), so
\(\operatorname{dec}^{\rm f,H}_\varepsilon(\mathcal H_k,\widehat M_k)
=\inf_{p\in\Delta(\Pi)}Q_k(p)\).
\end{definition}

\begin{lemma}[Original-scale DEC interface]
\label{lem:app-dec-interface}
For every continuous policy-coherent predictor \(M\) and every
\(H\)-bounded compact subclass
\(\mathfrak H\subseteq\mathfrak H_{\rm st}^{(H)}\),
\begin{equation}
\operatorname{dec}^{\rm f,H}_\varepsilon(\mathfrak H,M)
\le2H\sqrt{2(HSA+1)}\,\varepsilon.
\label{eq:app-fuzzy-dec-bound}
\end{equation}
\end{lemma}

\begin{proof}
Evaluate the same outcome laws with unit reward \(Y\) rather than
task-scale reward \(HY\), and divide every candidate and predicted value by
\(H\).  The feasible Hellinger-loss set is unchanged, while the objective
in Definition~\ref{def:learning-fuzzy-dec} is divided by \(H\).  The
candidate class is then 1-bounded.  Proposition
\ref{prop:app-predictor-continuity} verifies that the predictor produced by our
market belongs to \(\mathfrak M_{\rm adm}\).  For a generic
\(M\in\mathfrak M_{\rm adm}\), policy coherence is part of the definition:
for every query \(\pi\), define at every post-action history \(\eta_h\) ending
in \(s_h\)
\[
\sigma_h^\pi(\cdot\mid\eta_h,a)
\coloneqq M_{\pi,h}(s_h,a).
\]
Finite-horizon recursion then gives exactly
\(M_\pi=\sigma^\pi\bowtie\pi\), including at null histories.  Finally,
Lemma~\ref{lem:app-class-regularity} verifies the Borel regularity needed for
the mixtures below.

We next adapt the terminal convention before invoking the imported bound.
Define
\(T_{\rm term}(y,s')=(y,s_\dagger)\) at stage \(H\)
and leave earlier outcomes unchanged.  For a generic candidate, define
\[
(T_{\rm term\#}N)_H(s,a)
\coloneqq
\{T_{\rm term\#}\nu:\nu\in N_H(s,a)\};
\]
we do not identify the original row with the inverse image of its coarsening.
Because \(T_{\rm term}\) retains \(Y\), minimizing expected terminal reward
over the image is identical to minimizing it over the original row.  Backward
induction therefore preserves all candidate values and predicted returns, and
pushing forward an \(H\)-bounded selector preserves boundedness.

For any deterministic map \(T\), Cauchy--Schwarz on every fiber gives
Hellinger data processing:
\begin{equation}
D_{\rm H}^2(T_\#p,T_\#q)\le D_{\rm H}^2(p,q).
\label{eq:app-hellinger-data-processing}
\end{equation}
Indeed, the affinity after coarsening is
\[
\sum_y
\sqrt{\sum_{T(o)=y}p(o)\sum_{T(o)=y}q(o)}
\ge\sum_o\sqrt{p(o)q(o)}.
\]
Taking the infimum over \(q\in C\) gives
\[
d_{T_\#C}(T_\#p)\le d_C(p).
\]
Let \(L_{\rm AK}(T_{\rm term\#}M,T_{\rm term\#}N,\pi)\) denote the loss evaluated on the terminal-coarsened outcome
space.  The preceding rowwise inequality and the unchanged predicted
occupancies imply
\begin{equation}
L_{\rm AK}(T_{\rm term\#}M,T_{\rm term\#}N,\pi)
\le L(M,N,\pi).
\label{eq:app-loss-comparison}
\end{equation}
Thus, for each fixed policy law, a subprobability witness feasible under our
full-row loss pushes forward to one feasible under the coarsened loss with the
same unit objective.  Our inner supremum is therefore at most the coarsened
inner supremum over the image class.  The image is a compact 1-bounded subclass
of the finite tabular robust models covered by Theorem~4 of
\cite{appel2025regret}, so its DEC is at most their universal unit-scale
coefficient.

Finally prepend their deterministic one-action, zero-reward, zero-loss time-zero
row.  This changes no value or loss and accounts for the \(+1\) in
\(HSA+1\).  The imported theorem therefore bounds the unit coefficient by
\(2\sqrt{2(HSA+1)}\,\varepsilon\).  Multiplying the objective by \(H\)
gives equation~\eqref{eq:app-fuzzy-dec-bound}.  This positive-homogeneity
step, the generic terminal pushforward, and the dummy-row bookkeeping are
our adapters; the numerical unit-scale coefficient is the cited result.  No
task reward or observation is normalized in the algorithm.
\end{proof}

\subsection{E2D and the approximate-oracle extension}
\label{app:learning-e2d}

\cite[Theorem~1]{appel2025regret} give the E2D master inequality for an
exact policy-selection oracle, and
\cite[Remark~4.1]{foster2021statistical} explain that generic E2D analyses
accommodate inexact minimizers.  For transparency, we derive the explicit
\(2K\theta\) contribution for our localized fuzzy objective and keep the
predictability and measurability requirements visible.
Neither that theorem nor the general DEC formalism of
\cite{foster2021statistical} supplies an efficient oracle, or a measurable
approximate-selection procedure, for our localized stage-tied class.
Assumption~\ref{ass:decision-oracle} explicitly supplies the predictable
measurable \(p_k\); Lemma~\ref{lem:app-class-regularity} ensures that its
objective is Borel, and the following lemma records exactly what follows
conditional on this oracle assumption.

\begin{lemma}[Predictable approximate-oracle E2D]
\label{lem:app-approximate-e2d}
Let \(N^\circ\in\mathfrak H_{\rm st}^{(H)}\) be a fixed candidate with
\(v_{N^\circ}^\star=v^\star\), let \(\mathcal H_k\) be exactly the localized
class in equation~\eqref{eq:learning-localized-class}, and let
\(\varepsilon=\sqrt{\beta/K}\).  Suppose the predictable law \(p_k\) obeys
the approximate-oracle inequality
\[
Q_k(p_k)\le\inf_{p\in\Delta(\Pi)}Q_k(p)+\theta
\]
and that \(N^\circ\) satisfies, each with probability at least \(1-\delta\),
\begin{align}
\sum_{k=1}^K\mathbb E_{\pi\sim p_k}
L(\widehat M_k,N^\circ,\pi)&\le\beta,
\label{eq:app-e2d-inaccuracy}\\
\sum_{k=1}^K\mathbb E_{\pi\sim p_k}
[\widehat v_k(\pi)-m_k(\pi)]&\le\alpha.
\label{eq:app-e2d-optimism}
\end{align}
On the event in equation~\eqref{eq:app-e2d-inaccuracy}, nonnegativity of the
loss and the prefix definition of \(\mathcal H_k\) imply
\(N^\circ\in\mathcal H_k\) for every \(k\).  Then
\begin{equation}
\mathbb E\Reg_K
\le
2K\sup_{M\in\mathfrak M_{\rm adm}}
\operatorname{dec}^{\rm f,H}_{\varepsilon}
(\mathfrak H_{\rm st}^{(H)},M)
+2K\theta+\alpha+2HK\delta.
\label{eq:app-e2d-master}
\end{equation}
\end{lemma}

\begin{proof}
For \(\theta=0\), this is the task-scale on-policy specialization of
Theorem~1 of \cite{appel2025regret}: one round is one episode, a decision
is a Markov policy, and the observation is the coded public trajectory.
On the estimation event, nonnegativity of loss and the prefix form of
\eqref{eq:app-e2d-inaccuracy} place \(N^\circ\) in every
\(\mathcal H_k\).

We now isolate the effect of \(\theta\).  Put
\begin{align*}
\ell_k
&=\mathbb E_{\pi\sim p_k}L(\widehat M_k,N^\circ,\pi),\\
r_k
&=v_{N^\circ}^\star-
\mathbb E_{\pi\sim p_k}\widehat v_k(\pi),\\
w_k
&=\min\{1,\varepsilon^2/\ell_k\},
\end{align*}
with \(w_k=1\) when \(\ell_k=0\).  The subprobability measure of mass
\(w_k\) concentrated at \(N^\circ\) is feasible for the episode-\(k\) fuzzy
constraint because \(w_k\ell_k\le\varepsilon^2\).  If
\[
D=\sup_{M\in\mathfrak M_{\rm adm}}
\operatorname{dec}^{\rm f,H}_\varepsilon
(\mathfrak H_{\rm st}^{(H)},M),
\]
then localization gives
\(\operatorname{dec}^{\rm f,H}_\varepsilon(\mathcal H_k,\widehat M_k)
\le D\).  Because \(N^\circ\in\mathcal H_k\), approximate minimization
therefore gives
\begin{equation}
w_kr_k\le D+\theta.
\label{eq:app-approximate-witness}
\end{equation}
The zero subprobability is feasible, so \(D\ge0\).  If \(r_k<0\), the desired
upper bound below is automatic.  If \(r_k\ge0\), divide
equation~\eqref{eq:app-approximate-witness} by \(w_k>0\).  In both cases,
\[
r_k
\le(D+\theta)w_k^{-1}
\le(D+\theta)(1+\ell_k/\varepsilon^2).
\]
Summing and using
\(\sum_k\ell_k\le\beta=K\varepsilon^2\) gives
\begin{equation}
\sum_{k=1}^Kr_k\le2K(D+\theta).
\label{eq:app-approximate-oracle-sum}
\end{equation}
For every \(k\),
\[
r_k+\mathbb E_{\pi\sim p_k}
[\widehat v_k(\pi)-m_k(\pi)]
=
\mathbb E_{\pi\sim p_k}
[v_{N^\circ}^\star-m_k(\pi)].
\]
Thus adding equation~\eqref{eq:app-e2d-optimism} gives the good-event bound
\(2KD+2K\theta+\alpha\) on the sum of conditional expected factual-return
gaps.

On the union of the two failure events, each true conditional regret gap is
at most \(H\): \(v_{N^\circ}^\star=v^\star\le H\) and
\(m_k(\pi)\ge0\).  Thus the failure contribution is at most
\(2HK\delta\).  This argument does not bound the predictor, which may be as
large as \(H^2\).  Finally, predictability of \(p_k\) gives
\[
\mathbb E\Reg_K
=
\mathbb E\sum_{k=1}^K\mathbb E_{\pi\sim p_k}
[v^\star-m_k(\pi)].
\]
Taking expectations proves
equation~\eqref{eq:app-e2d-master}.
\end{proof}

\subsection{Final substitution and security guarantee}
\label{app:learning-final-proof}

\begin{proof}[Proof of Theorem~\ref{thm:stage-tied-regret}]
Lemma~\ref{lem:app-surrogate-applicability} supplies a fixed true candidate
\(N^\circ\) for every episode.  This is where the assumptions that
\(B^\star\) and the physical kernel remain fixed are used.  The algorithm does
not need to know this candidate.  Lemma~\ref{lem:app-conversion} identifies
its task-scale comparator and factual returns with \(v^\star\) and
\(m_k(\pi)\), respectively.

If \(K=1\), the one-episode return-regret gap is at most \(H\).  Moreover,
\[
\Lambda_1(1)=\log(512H^2A)>1,
\]
and hence \(114H^2S\sqrt{A\Lambda_1(1)}\ge H\).  Consequently the minimum in
equation~\eqref{eq:learning-return-regret} equals its deterministic one-episode
branch \(H\).  This also avoids applying a confidence theorem with the invalid
input \(\delta=1\).

At the confidence level used by Algorithm~\ref{alg:stage-tied-e2d} for \(K\ge2\),
\begin{equation}
\Lambda_K(K^{-2})
=\log(128H^2A)+2\log(K+1)+4\log K,
\label{eq:learning-expanded-log}
\end{equation}
so \(\Lambda_K(K^{-2})=\Theta(\log K)\) for fixed \(H\) and \(A\).

Assume \(K\ge2\), set \(\delta=\theta=K^{-2}\), and abbreviate
\[
\Lambda=\Lambda_K(K^{-2}),
\qquad
L=\log(32K^2).
\]
If \(\Lambda\ge K\), the deterministic cap gives
\[
\mathbb E\Reg_K
\le HK
\le H^2S\sqrt{AK\Lambda}.
\]
This is only a finite-parameter fallback; equation~\eqref{eq:learning-expanded-log}
shows that \(\Lambda\) is normally logarithmic in \(K\) for fixed \(H,A\).

It remains to treat \(0<\Lambda<K\).  Theorem~\ref{thm:stage-tied-estimation}
allows \(\beta=144HS\Lambda\), and
Lemma~\ref{lem:app-optimism} allows
\[
\alpha=8H^2\{\sqrt{KL}+L\}.
\]
Apply Lemma~\ref{lem:app-approximate-e2d} and then
Lemma~\ref{lem:app-dec-interface}.  We obtain
\begin{align}
\mathbb E\Reg_K
&\le
4H\sqrt{2(HSA+1)\beta K}+\alpha+\frac{2(1+H)}K
\nonumber\\
&\le
8H\sqrt{HSA\,\beta K}+\alpha+\frac{4H}K.
\label{eq:app-main-substitution}
\end{align}
The second inequality uses \(HSA+1\le2HSA\) and \(H\ge1\).  Substituting
\(\beta=144HS\Lambda\) makes the first term
\begin{equation}
8H\sqrt{HSA\cdot144HS\Lambda\cdot K}
=96H^2S\sqrt{AK\Lambda}.
\label{eq:app-leading-term}
\end{equation}
Furthermore,
\[
\Lambda-L
=\log\!\left(4H^2A(K+1)^2K^2\right)>0,
\]
so \(L\le\Lambda\).  Because \(0<\Lambda<K\),
\(\Lambda\le\sqrt{K\Lambda}\).  Therefore
\begin{equation}
\alpha
\le8H^2\{\sqrt{K\Lambda}+\Lambda\}
\le16H^2S\sqrt{AK\Lambda}.
\label{eq:app-optimism-domination}
\end{equation}
Finally, \(K\ge2\) gives \(4H/K\le2H\), while
\(H^2S\sqrt{AK\Lambda}\ge H\); hence
\begin{equation}
\frac{4H}K\le2H^2S\sqrt{AK\Lambda}.
\label{eq:app-residual-domination}
\end{equation}
Combining equations~\eqref{eq:app-main-substitution}-
\eqref{eq:app-residual-domination} and using \(96+16+2=114\) proves
\[
\mathbb E\Reg_K
\le114H^2S\sqrt{AK\Lambda}.
\]
This is already in the original task units and proves the statistical
branch of equation~\eqref{eq:learning-return-regret}.  The independent
pathwise cap \(HK\) proves its minimum form.

For every episode, the realized response is feasible, so
\(m_k\ge W(\pi_k)\).  Therefore \(D_K\ge0\), and the algebraic identity
\[
\operatorname{Reg}^{\rm sec}_K
=\operatorname{Reg}_K+D_K
\]
holds pathwise.  Taking expectations and using the return bound gives the
universal security bound with \(+\mathbb E[D_K]\).  Independently, every
security-regret summand belongs to \([0,H]\), so security regret is at most
\(HK\) pathwise.  Taking the smaller proves
equation~\eqref{eq:learning-security-regret}.  If
\(m_k-W(\pi_k)\le\epsilon_k\), then
\(D_K\le\sum_k\epsilon_k\), which proves
  \begin{equation}
\mathbb E[\Regsec_K]
\le
\min\!\left\{HK,\,
\mathcal C_K+\sum_{k=1}^K\mathbb E[\epsilon_k]\right\}. 
\end{equation}
Exact worst responses have
\(D_K=0\).
\end{proof}

\end{document}